\documentclass[onefignum,onetabnum]{siamonline250211}
\usepackage{amsmath,algorithmic,amssymb}
\newsiamremark{remark}{Remark}
\newsiamremark{hypothesis}{Hypothesis}
\crefname{hypothesis}{Hypothesis}{Hypotheses}
\newsiamthm{claim}{Claim}
\newsiamremark{fact}{Fact}
\crefname{fact}{Fact}{Facts}

\newcommand{\Q}{\mathcal{Q}}
\newcommand{\R}{\mathbb{R}}

\DeclareMathOperator{\Tr}{Tr}
\newcommand{\tr}{\operatorname{trace}}

\newcommand{\new}[1]{\textcolor{black}{#1}}

\ifpdf
\hypersetup{
  pdftitle={Post-training Quantization of Neural Networks: A Theoretical Analysis of OPTQ and Qronos},
  pdfauthor={H. Zhang, S. Zhang, I. Colbert, and R. Saab}
}
\fi

\headers{Theoretical Analysis of OPTQ and Qronos}{H. Zhang, S. Zhang, I. Colbert, R. Saab}

\begin{document}

\title{Provable Post-Training Quantization: Theoretical Analysis of OPTQ and Qronos}

\author{Haoyu Zhang\thanks{Equal contribution, Department of Mathematics, UC San Diego 
  (\email{haz053@ucsd.edu}, \email{shz051@ucsd.edu}), corresponding author: Haoyu Zhang (\email{haz053@ucsd.edu})
  }
\and Shihao Zhang\footnotemark[1]
  \and Ian Colbert\thanks{Software Architecture, AMD
  (\email{ian.colbert@amd.com})}
\and Rayan Saab\thanks{Department of Mathematics and HDSI, UC San Diego 
  (\email{rsaab@ucsd.edu})}}

\maketitle

\begin{abstract}
Post-training quantization (PTQ) has become a crucial tool for reducing the memory and compute costs of modern deep neural networks, including large language models (LLMs). Among PTQ algorithms, the OPTQ framework—also known as GPTQ—has emerged as a leading method due to its computational efficiency and strong empirical performance. Despite its widespread adoption, however, OPTQ lacks rigorous quantitative theoretical guarantees. This paper presents the first quantitative error bounds for both deterministic and stochastic variants of OPTQ, as well as for Qronos, a recent related state-of-the-art PTQ algorithm. We analyze how OPTQ's iterative procedure induces quantization error and derive non-asymptotic $\ell_2$ error bounds that depend explicitly on the calibration data and a regularization parameter that OPTQ uses. Our analysis provides theoretical justification for several practical design choices, including the widely used heuristic of ordering features by decreasing norm, as well as guidance for selecting the regularization parameter. For the stochastic variant, we establish stronger $\ell_\infty$ error bounds, which enable control over the required quantization alphabet and are particularly useful for downstream layers and nonlinearities. Finally, we extend our analysis to Qronos, providing new theoretical bounds, for both its deterministic and stochastic variants, that help explain its empirical advantages.
\end{abstract}



\section{Introduction}\label{sec:intro}
Recent breakthroughs in deep neural networks—most notably large language models (LLMs)—have {introduced} massive computational and memory demands. These costs have spurred interest in model compression methods that make LLM deployment more practical \cite{xu2023survey, zhu2023survey}. A key compression method is quantization, which reduces the number of bits used to represent each weight or activation (their \emph{bit width}), thereby lowering the requirements for {storage, movement, and computation}. Quantization methods achieve this reduction by simply replacing the real-valued weights (or activations) by elements from a finite set. {Quantization approaches can}
be divided into two categories: (1) quantization-aware training (QAT) \cite{Jacob_2018_CVPR, xi2023training, zhang2022learning}, where quantized models are learned during training via some variant of gradient descent; and (2) post-training quantization (PTQ) \cite{quip, frantar2022gptq, awq, zhang2023post}, where quantized models are constructed after training. 
{Unlike QAT,} PTQ is usually back-propagation-free and adjusts a pre-trained model in one pass. {Therefore,} it {incurs significantly less computational overhead.} Moreover, it typically only requires a small calibration dataset. 
As such, it is widely adopted \cite{gholami2022survey, zeroquant} and it now enables few-bit LLM inference in practice.

\subsection{Contributions} \label{sec:contributions}

We present the first quantitative error guarantees for post-training quantization (PTQ) algorithms built on the widely used OPTQ framework—also known as GPTQ \cite{frantar2022gptq}. OPTQ has become the de-facto PTQ method across diverse neural network architectures \cite{qubitium2024gptqmodel}. Consequently, nearly all new quantization schemes (e.g., \cite{quarot,quip, spinquant, zhang2024magr}) benchmark against it, underscoring its status as the standard PTQ baseline. Thus, we focus on both deterministic and stochastic variants of OPTQ, as well as Qronos \cite{zhang2025qronos}, a recent related state-of-the-art algorithm. 

 The OPTQ algorithm maps a weight vector $w \in \mathbb{R}^N$ to a vector $q \in \mathcal{A}^N$, where $\mathcal{A} \subset \mathbb{R}$ is a finite quantization alphabet, by targeting the error measured against a fixed calibration data matrix $X \in \mathbb{R}^{m \times N}$, i.e., by targeting $\|Xw - Xq\|_2$.  It proceeds iteratively, alternating between first quantizing a coordinate of $w$, then updating the remaining unquantized coordinates to compensate for the induced error.  This greedy strategy is applied  to all the weight vectors in a layer and repeated layer-wise. It is also worth noting that OPTQ typically involves working with a regularized version of the covariance matrix $X^TX + \lambda I$, where the regularization parameter $\lambda$ helps stabilize the algorithm. OPTQ has proven highly effective in practice, but despite its success and ubiquity, rigorous quantitative analyses of OPTQ’s accuracy have been lacking. We close this gap by deriving non-asymptotic bounds on its quantization error, characterizing its dependence on  $N$, on properties of the calibration data $X$, and on the choice of  regularization parameter $\lambda$. We provide: 
\smallskip
\paragraph{An analysis of OPTQ with $\ell_2$ error bounds}
We establish the first  error bounds for OPTQ:
\begin{itemize}
\item We characterize how the error in OPTQ iteratively evolves in \cref{lemma:OPTQ_error_recur}. 
\item Using this characterization,  we derive \emph{deterministic $\ell_2$ bounds} (\cref{thm:l2 full version,thm:l2 simp ver}) that reveal how  the error depends on conditioning of sub-matrices of the calibration data $X$, and on $\lambda$. 
\item \new{\Cref{thm:l2 simp ver} also yields an upper bound on the expansion of the grid size relative to the dynamic range of the weights, and \cref{Appendix:adversarial} shows that this bound is tight via a matching example. Because this worst case scaling can be unfavorable, it partly motivates the use of stochastic rounding to improve robustness to adversarial data and weight configurations.}

\item As \new{another} by-product, we rigorously justify a heuristic that is widely used in practice but previously lacked formal support: namely, the strategy of ordering features (columns of $X$) by decreasing norm before quantization (\cref{rem:decreasing}).
\end{itemize}
\smallskip
\paragraph{A stochastic variant of OPTQ with $\ell_\infty$ error bounds} We also  analyze a stochastic rounding variant of OPTQ and prove \emph{stronger $\ell_\infty$ bounds} (\cref{thm:l inf full ver}), thereby obtaining  \emph{explicit control of the required alphabet size} for quantization (\cref{remark:finte alphabet}).
The  stochastic version is motivated by overcoming three challenges:
\begin{itemize}
\item When quantizing activations, $Xq$ (or some Lipschitz function of $Xq$) must also be quantized since it becomes the input to the next layer. Controlling  $\|Xw - Xq\|_\infty$ bounds the required bit-width
for the next layer's activation quantization.
\item Deterministic OPTQ does not provide direct $\ell_\infty$ control on the updated weights, making it difficult to bound the required bit width for weight quantization. The stochastic variant overcomes this limitation.
\item Many neural network layers involve nonlinearities—such as softmax—where output ranking is sensitive to large coordinate errors. An $\ell_2$ bound may look small yet fail to capture or prevent ranking flips, while an $\ell_\infty$ bound can provide guarantees, especially if there is a gap between the  largest entries. 
\end{itemize}
\smallskip
\paragraph{New theoretical results for {Qronos}}
We extend our framework to analyze {Qronos} \cite{zhang2025qronos}, a recent PTQ method with \new{state-of-the-art} empirical results. Our analysis provides new $\ell_2$ and $\ell_\infty$ error bounds (see \cref{sec:Qronos theory}) that help explain its superior performance in practice.

\subsection{Preliminaries and Notation}\label{sec:notation}
Before presenting our theoretical results, let us formalize notation and review some necessary preliminaries, including those associated with neural networks and quantization.

\new{When quantizing $W$, we use $X \in \mathbb{R}^{m \times N}$ to denote the input calibration dataset of $m$ samples (\textit{e.g.}, tokens) for the layer, resulting from the original pre-trained model, and $\widetilde{X} \in \mathbb{R}^{m \times N}$ to denote the input calibration dataset coming from the partially quantized model. Given a vector $v \in \mathbb{R}^n$, we use $v_i$ for its $i$-th entry, $v_{\geq j}$ for the subvector $(v_j, \dots, v_n)^\top$, and we define $v_{\leq j}$ analogously. $\|v\|$ is the Euclidean norm of $v$.  Given a matrix $A \in \mathbb{R}^{m \times n}$, we use $A_i$ to denote its $i$-th column. We use $A_{\geq j}$ to denote the submatrix $(A_j, \dots, A_n)$.} We denote the column space of a matrix $A$ by  $\text{col}(A)$.  $P_A$ is the orthogonal projection onto $\text{col}(A)$ given by $P_A=AA^\dagger$ and $P_{A^\perp}$ is the projection onto its orthogonal complement given by $P_{A^\perp} = I - AA^\dagger$, where $\dagger$ represents pseudo inverse. 
Throughout this paper, all indices start from 1.

An $L$-layer multilayer perceptron (MLP) is a map that composes affine functions and non-linear activation functions that act component wise: 
\[
  \mathbf{\Phi}\colon\R^{N_0}\longrightarrow\R^{N_L},
  \qquad
  \mathbf{\Phi}(x)
  \;=\;
  \phi^{[L]}\!\circ A^{[L]}\!\circ\dots\circ\phi^{[1]}\!\circ A^{[1]}(x).
\]
Here, for each layer $\ell=1,\dots,L$ we have the affine functions 
\[
  A^{[\ell]}(z)\;=\;{W^{[\ell]}}^{\!\top}z+b^{[\ell]},
  \qquad
  W^{[\ell]}\in\R^{N_{\ell-1}\times N_\ell},
  \;\;
  b^{[\ell]}\in\R^{N_\ell},
\] and the activation functions $\phi^{[\ell]}\colon\R^{N_\ell}\to\R^{N_\ell}$.
We extend the definition of $\mathbf{\Phi}$ to matrix inputs $X \in \R^{m \times N_0}$ by applying it row-wise, that is, 
\[
  \mathbf{\Phi}(X) := 
  \begin{bmatrix}
    \mathbf{\Phi}(\mathbf{x}_1)^\top \\
    \vdots \\
    \mathbf{\Phi}(\mathbf{x}_m)^\top
  \end{bmatrix}
  \in \R^{m \times N_L},
  \qquad \text{where each } \mathbf{x}_i \in \R^{N_0} \text{ is a row of } X.
\]
Let $X_0\in\R^{m\times N_0}$ contain $m$ input samples as rows (\textit{e.g.}, tokens in LLMs). Transformers replace some of the layers in MLPs with ``attention mechanisms," non-linear functions that do not operate elementwise. In this context, for example, self-attention maps  $X \in \R^{m \times N}$  to
\[
  \mathrm{Attention}(X) \;=\; \mathrm{softmax}\left(\frac{XW_Q (XW_K)^\top}{\sqrt{N}}\right) XW_V \ \in \R^{m\times N},
\]
where $W_Q, W_K, W_V \in \R^{N \times N}$ are learned weight matrices for ``queries", ``keys", and ``values", respectively\footnote{When applied to a matrix $Z \in \R^{m \times m}$, the softmax function acts row-wise. Each row is exponentiated element-wise and normalized to sum to one.}.

For our purposes in this paper, the important point ---regardless of whether one is dealing with an attention mechanism or an MLP structure--- is that products of the form $XW$ are ubiquitous, and the corresponding weight matrices $W$ need to be quantized via algorithms that preserve these products. 

 \subsection{Quantization preliminaries}\label{sec:notations}
Before introducing quantization in more detail, let us note that in most PTQ methods, weight matrices $W^{[1]},\ldots,W^{[L]}$ are quantized sequentially, one layer at a time.  
Define the truncated networks
obtained from the original and quantized models after layer \(\ell\), and set the corresponding activation matrices
\[
  X^{[\ell]}
  := \mathbf{\Phi}^{[\ell]}(X_0)
  =  \phi^{[\ell]}\!\bigl(X^{[\ell-1]}W^{[\ell]}\bigr),
  \qquad
  \widetilde{X}^{[\ell]}
  := \widetilde{\mathbf{\Phi}}^{[\ell]}(X_0)
  =  \phi^{[\ell]}\!\bigl(\widetilde{X}^{[\ell-1]}\widetilde{W}^{[\ell]}\bigr),
\]
with \(X^{[0]}=\widetilde{X}^{[0]}:=X_0\).
The matrices \(X^{[\ell-1]}W^{[\ell]}\) and \(\widetilde{X}^{[\ell-1]}\widetilde{W}^{[\ell]}\) are the associated pre-activations. Because our analysis focuses on a single, generic layer, we suppress the layer superscript and write  
\(XW\) \new{(or $\widetilde{X}W$)} for the full pre-activation matrix and \(Xw\) \new{(or $\widetilde{X}w$)} for the pre-activation of a single output channel. \new{Here, $X, \tilde{X} \in R^{m \times N}$} stacks \(m\) samples (e.g., tokens) as rows, \(W\in\R^{N\times N'}\) is the weight matrix, and \(w\) denotes one of its columns (i.e.,\ a single  channel). 

A PTQ algorithm replaces $W^{[\ell]}\in \mathbb{R} ^ {N_{\ell-1}\times N_{\ell}}$ by $Q^{[\ell]}\in \mathcal{A} ^ {N_{\ell-1}\times N_{\ell}}$ and uses some possibly scaled, shifted, or {truncated} variant of the  finite alphabet (or quantization grid) $$\mathcal{A}=\mathcal{A}^{\delta}_b:= \left\{ \pm k\delta : k = -2^{b-1},..., -1, 0, 1,..., 2^{b-1} \right\},$$ 
with $|\mathcal{A}|=2^{b}+1$. If the alphabet used is symmetric about $0$, we call it symmetric quantization. Otherwise, we call it asymmetric quantization\footnote{\new{Although we focus mainly on symmetric alphabets in the later discussion of the finite-alphabet setting, the same arguments extend directly to asymmetrically ranged weights and asymmetric alphabets.}}. 
Similarly, we define the infinite alphabet $\mathcal{A}=\mathcal{A}^{\delta}:= \{ \pm k\delta : k \in \mathbb{Z} \}$. For each alphabet, we associate
a memoryless scalar quantizer (MSQ) $\mathcal{Q}:\mathbb{R}\rightarrow \mathcal{A}$ given by
$
\mathcal{Q}(z):=\arg\min_{p \in \mathcal{A}}|z-p|,
$ which essentially executes a ``round to nearest" (RTN) operation. 
In the case of the infinite alphabet, this becomes
$
\mathcal{Q}(z)=\delta \mathrm{sign} (z) \left| \lfloor \frac{z}{\delta}+\frac{1}{2} \rfloor \right|.
$

We  define the unbiased stochastic scalar quantizer $\mathcal{Q}_{stoc}:\mathbb{R}\rightarrow \mathcal{A}$, which randomly rounds a real number $z\in [k\delta, (k+1)\delta]$ either to $k\delta$ or to  $(k+1)\delta$  such that $\mathbb{E}[\mathcal{Q}_{stoc}(z)]=z$. Specifically
$$
\mathcal{Q}_{stoc}(z):=\left\{
\begin{array}{lr}
    \lfloor\frac{z}{\delta}\rfloor\delta & \text{with probability $p$},\\
    (\lfloor\frac{z}{\delta}\rfloor+1)\delta & \text{with probability $1-p$},
\end{array}
\right.
$$
where $p=1-\frac{z}{\delta}+\lfloor\frac{z}{\delta}\rfloor$.

{The latest} post-training quantization (PTQ) {pipelines often {comprise}} two complementary {stages}: {transforms and rounding}.
\paragraph{Transforms}
{Quantization transforms aim to} modify the weights and activations of a model to make them more amenable to quantization. {The most popular transformations include channel rescaling, matrix rotations, and model expansions.}
Channel rescaling balances per-channel ranges prior to quantization by replacing $X\mapsto XD^{-1}$, $w\mapsto Dw$ for some {optimized} diagonal matrix $D$ {before} quantizing the resulting weights (and possibly activations) \cite{awq,nagel2019data,omniquant,smoothquant}.
{Matrix rotation techniques} replace the diagonal matrix by orthogonal rotations (random, Hadamard, or learned on the Stiefel manifold) to control the magnitude across dimensions
(e.g., \cite{quarot,quip,spinquant,quip2}).
{Model expansion techniques counterintuitively increase parameter count post-training to ultimately reduce parameter volume (i.e., model size \( \times \) bit width) by further reducing parameter bit width \cite{adepu2024framequant, franco2025improving}.}
Meanwhile, MagR reduces dynamic range by minimizing the $\ell_\infty$ norm of the weights
\cite{zhang2024magr}.

\paragraph{Rounding}
Early LLM {quantization} methods fixed {the quantization grid} heuristically, then rounded {weights} to the
nearest grid point \cite{dettmers2022gpt3,zeroquant}.
Greedy layer-wise algorithms such as OBQ, OPTQ, GPFQ,  and Qronos quantize a weight vector  sequentially to approximately minimize \new{layer-wise} reconstruction error
\cite{obq,frantar2022gptq,lybrand2021greedy,zhang2023post,zhang2025qronos}.
Some recent work enriches the grid itself, for example, employing vector quantizers
\cite{quip2}, which can result in lower bit-rates. On the other hand, vector quantizers typically keep a code-book in memory, adding storage and extra look-up operations that can reduce inference speed. 
Moreover, performing vector quantization entails solving combinatorial optimization problems whose complexity increases exponentially with dimension, increasing the computational cost of the quantization itself, and limiting compute acceleration opportunities during inference.

\section{Background and Related Work}

Before introducing OPTQ and Qronos \cite{frantar2022gptq, zhang2025qronos}, let us first describe the core problem these quantization algorithms aim to solve, then review existing theoretical guarantees for PTQ methods. \new{Both OPTQ and Qronos seek to minimize the layer-wise reconstruction error. For OPTQ which uses one calibration matrix, this either takes the form
\(
\underset{Q \in \mathcal{A}^{N \times N'}}{\min} \ \|{X}W - {X}Q\|_F^2,
\) or more commonly in practice,
\begin{gather}\label{optq_error}
\underset{Q \in \mathcal{A}^{N \times N'}}{\min} \ \|\widetilde{X}W - \widetilde{X}Q\|_F^2.
\end{gather}
Meanwhile, Qronos, like GPFQ \cite{lybrand2021greedy, zhang2023post} before it, seeks to minimize
\begin{gather}\label{qronos_error}
\underset{Q \in \mathcal{A}^{N \times N'}}{\min} \ \|XW - \widetilde{X} Q\|_F^2,
\end{gather}
where $X$ and $\widetilde{X}$ are defined in \cref{sec:notations} as the activations from the original model and the (possibly quantized) activations from the quantized model, respectively.
} 
Both objectives are instances of integer least-squares problems, which are NP-hard \cite{hassibi2002expected}. As such, efficient algorithms can only approximate their solutions, differing, for example,  in how they balance accuracy and computational cost. Indeed, many  PTQ methods share this goal, including \cite{adaquant, lybrand2021greedy, adaround}.

\new{Since OPTQ only uses one calibration matrix, we will slightly abuse our notation for simplicity to always use $X$ throughout our analysis of OPTQ in \Cref{sec:OPTQ l2 error bound} and \Cref{sec:OPTQ linf error bound}.}

\subsection{Existing Theoretical Guarantees for Quantizing Neural Networks}
Despite an extensive body of research on post-training quantization methods, most well-known algorithms lack theoretical guarantees, with few exceptions. \new{For example, \cite{quip} provides an analysis showing that quantization benefits from incoherent weight
and Hessian matrices.}
 Another exception is a research thread focusing on the GPFQ algorithm and its variants \cite{lybrand2021greedy, zhang2024unified, zhang2023spfq, zhang2023post}. In \cite{lybrand2021greedy}, an error bound for  ternary weight quantization  is derived under the assumption that the rows of $X$ are independently sampled from a Gaussian distribution. Then, \cite{zhang2023post} used a different proof technique that allowed extending the results to more general quantization grids and a wider range of data distributions, including Bernoulli and Gaussian clusters. Subsequently, \cite{zhang2023spfq} introduced stochastic rounding to completely remove the need for randomness assumptions on $X$. These results applied to arbitrary data matrices $X$ and sufficiently large alphabets. The proof technique was further extended in \cite{zhang2024unified} to handle cases when the quantization grid has a given finite size and to incorporate pruning. These works prove explicit error bounds as a function of $X$ and the various dimension parameters, as we do in this work for OPTQ and Qronos. \new{Notably, the results in previous works all rely on stochasticity, either in data distribution, weights, or in the quantizer. 
Our deterministic $\ell_2$ error bounds do not require any such assumptions, and provide an \emph{equality} characterization for the error dynamics of OPTQ and Qronos. Our $\ell_\infty$ error bounds rely on stochastic rounding and extend the proof techniques in \cite{zhang2024unified, zhang2023spfq} to handle the more complicated update rules that involve two least-squares solvers in each iteration (\cref{thm:OPTQ2})}.

\begin{algorithm}[t]
\caption{OPTQ: Quantize a layer \( W \) to \( Q \)}
\begin{algorithmic}[1]\label{OPTQ}
\STATE  \( H^{-1} = (X^\top X+\lambda I)^{-1}=LL^\top \) \hfill {Perform Cholesky decomposition}
\FOR{every column $w$ in $W$ (in parallel)}
\STATE $w^{(0)} {=w}$
\FOR{$t = 1$ to $N$} 
        \STATE
         $q_{t}=\mathcal{Q}(w^{(t-1)}_t)$ \label{eq1} \hfill {Quantize current weight}
        \STATE \( w^{(t)}_{\geq t+1}= w^{(t-1)}_{\geq t+1}+(q_{t}-w^{(t-1)}_t)\frac{L_{\geq t+1,t}}{L_{tt}}  \) \label{eq2} \hfill {Update remaining weights}
    \ENDFOR
\ENDFOR
\STATE \textbf{return} every \( q \) in \( Q \) \hfill The matrix of quantized neurons
\end{algorithmic}
\end{algorithm}

\subsection{An Introduction to OPTQ}\label{subsec:OPTQ intro}

As discussed in \cref{sec:intro}, OPTQ is a widely used baseline in many recent works on post-training quantization (PTQ). OPTQ and related algorithms \cite{obq, frantar2023sparsegpt, frantar2022gptq} build on a framework that traces back to the Optimal Brain Surgery (OBS) approach \cite{hassibi1993optimal}, where pruning and quantization are performed iteratively by solving a small optimization problem at each step. More specifically, denoting the “Hessian” by $H = X^\top X$ and letting $\delta_w$ represent the update to the weight vector, the OBS pruning step solves:
\[
\underset{\delta_w}{\min} \ \frac{1}{2} \delta_w^\top H \delta_w \quad \text{subject to} \quad e_p^\top \delta_w + w_p = 0, \quad \delta_w|_F = 0,
\]
where $e_p$ is the standard basis vector selecting the $p$-th coordinate to prune, and $\delta_w|_F = 0$ enforces no change to already-fixed coordinates (see \cite{hassibi1992second}). This paradigm underlies both modern pruning strategies and quantization methods such as OPTQ.
Similarly, for quantization, each step involves solving
\[
\underset{q \in \mathcal{A}}{\min} \left\{ \ \underset{\delta_w}{\min} \ \frac{1}{2} \delta_w^\top H \delta_w \ \text{subject to} \ e_p^\top \delta_w + w_p = q, \quad \delta_w|_F = 0 \ \right\},
\]
where $\mathcal{A}$ is the quantization alphabet (see \cite{obq}).

 In the pruning case, this constrained quadratic problem admits a closed-form solution via the stationary point of its Lagrangian. In the quantization setting, the inner problem remains convex and can be solved in the same way, but since $q$ must lie in a discrete set $\mathcal{A}$, one must evaluate the objective over all possible values in $\mathcal{A}$ and select the minimizer. This leads to a natural greedy algorithm that quantizes one coordinate at a time while accounting for its impact on the overall output. With a few variations to improve efficiency and stability, this turns out to be equivalent to the iterations in OPTQ (\cref{OPTQ})\footnote{We follow our convention established in \cref{sec:notation} by using $XW$ as layer output where each neuron $w$ is a column of $W$ and $X_j, \, j=1,2,..., N$ represents features in \cref{OPTQ}. This notation is different from \cite{frantar2022gptq} where the authors were using $\widetilde{W}X$ for layer output and each neuron is a row of $\widetilde{W}$.}.

The first notable modification is that OPTQ uses the Cholesky factor $L$ in the decomposition $H^{-1}=LL^T$ in place of $H^{-1}$ itself as it gives a computationally equivalent output when the Cholesky decomposition exists. The second variation in \cref{OPTQ}, which is more critical from a mathematical perspective, 
is the introduction of a ``dampening" term $\lambda I$, added to $X^TX$ when computing the inverse Hessian to mitigate numerical instability\footnote{Another potential variation (called lazy batch updates in \cite{frantar2022gptq}) involves processing the weights in blocks of size $B$ to enhance the compute-to-memory-access ratio while preserving the algorithm's mathematical equivalence to the $B=1$ case. Thus, without loss of generality {we ignore $B$ in our mathematical analysis of OPTQ throughout this paper}.}.

\begin{algorithm}[t]
\caption{Qronos: Quantize a layer \( W \) to \( Q \)}
\label{Qronos++}
\begin{algorithmic}[1]
\STATE $H^{-1} = (\widetilde{X}^\top \widetilde{X}{+\lambda I})^{-1} = LL^\top $ \hfill {Perform Cholesky decomposition}
\FOR{every column $w$ in $W$ (in parallel)}
\STATE
 $w^{(0)}{=w}$ 
\STATE 
$q_{1}=\mathcal{Q}\left(\dfrac{\widetilde{X}_1^{\top} (Xw - \widetilde{X}_{\geq2}w^{(0)}_{\geq 2})}  {\|\widetilde{X}_1\|_2^2}\right)$
\hfill{Quantize first weight}
    \STATE $w^{(1)}_{\geq 2}=\widetilde{X}_{\geq2}^{\dagger}\left(Xw - q_{1}\widetilde{X}_{1}\right)$ \hfill {Update remaining weights}

\FOR{$t = 2$ to $N$} 
    \STATE $q_{t}= \Q(w^{(t-1)}_{t})$ \hfill {Quantize current weight}
    \STATE \( w^{(t)}_{\geq t+1} 
    = w^{(t-1)}_{\geq t+1}- L_{\geq t+1,t}\cdot (w^{(t-1)}_{t} - q_{t})/L_{tt}  \) \hfill {Update remaining weights}
\ENDFOR
\ENDFOR
\STATE \textbf{return} every $q$ in $Q$ \hfill {The matrix of quantized neurons}
\end{algorithmic}
\end{algorithm}
\subsection{Qronos}
We now introduce Qronos \cite{zhang2025qronos}, as our theoretical analysis extends to this algorithm as well. Qronos is a recently proposed state-of-the-art PTQ algorithm that sequentially rounds and updates neural network weights. It demonstrably subsumes and surpasses OPTQ 
via explicitly correcting quantization error in both the weights and activations of previous layers  while diffusing error into future weights. Qronos is derived from a disciplined mathematically interpretable framework, discussed in more details in \cref{sec:Qronos theory}. It also has a computationally efficient implementation (\cref{Qronos++}) {that leverages} existing optimizations {proposed for} OPTQ, such as Cholesky decomposition and block-level error diffusion. It was shown in \cite{zhang2025qronos} that Qronos outperforms OPTQ {including, for example,} on Llama 3 models \cite{grattafiori2024llama} across a range of bit budgets.

\section{$\ell_2$-Norm Error Analysis of OPTQ}\label{sec:OPTQ l2 error bound}
Our goal in this section is to bound the reconstruction error $\|Xw - Xq\|_2$ associated with OPTQ (\cref{OPTQ}). 

We denote the full state of the algorithm after step $t$ by the vector
\(
w^{(t)} = (q_{\leq t}, w^{(t)}_{\geq t+1})\in \mathcal{A}^{t} \times \mathbb{R}^{N-t},
\)
with the initialization $w^{(0)} = w\in\R^{N}$ and final output $w^{(N)} = q\in\mathcal{A}^{N}$. 
Let $X \in \mathbb{R}^{m \times N}$ be a calibration data matrix with columns $X = \begin{pmatrix} X_1 & \dots & X_N \end{pmatrix}$, and let $w = (w_1, \dots, w_N)^\top \in \mathbb{R}^N$ be the weight vector to be quantized. Running OPTQ (\cref{OPTQ}) on $X$ with regularization parameter $\lambda > 0$, i.e., using the Hessian $H = X^\top X + \lambda I$, is equivalent to applying \cref{OPTQ} without regularization to the augmented matrix
\[
\widehat{X} = \begin{pmatrix} X \\ \sqrt{\lambda} I \end{pmatrix}.
\]
This equivalence follows directly from the identity $X^\top X + \lambda I = \widehat{X}^\top \widehat{X}$. Notably, $\widehat{X}$ is always full rank with more rows than columns, regardless of whether $X$ itself is full rank or whether $m \geq N$. This justifies our initial focus on the unregularized case $\lambda = 0$ with full-rank $X$.

In \cref{subsec:OPTQ error update}, we begin by reviewing the equivalence between the least-squares and Cholesky formulations of OPTQ under the assumption that $X \in \mathbb{R}^{m \times N}$ has full column rank (i.e., $m \geq N$ and $\operatorname{rank}(X) = N$). This allows for a clean derivation of the OPTQ error dynamics and leads to explicit error bounds, first in the unregularized case $\lambda = 0$, and then for general $\lambda > 0$.

Then, in \cref{sec:practical}, we use these theoretical results to provide insight into several empirical practices in the literature. These include the common strategy of sorting columns of $X$ by decreasing norm, the selection of the regularization parameter $\lambda$ and its role in controlling the alphabet size and the generalization error, and the practical advantage of OPTQ over simple round-to-nearest methods such as \new{memoryless} scalar quantization (MSQ).


\subsection{OPTQ Error Dynamics and Bounds}\label{subsec:OPTQ error update} 

Recall that at the end of the $t$-th iteration, OPTQ has replaced the original weight vector $w$ with the partially quantized vector $w^{(t)} = (q_{\leq t}, w^{(t)}_{\geq t+1})$. 
So, it is natural to define the error at step $t$ as
\begin{equation}\label{eq:err_def}
e_t = Xw - Xw^{(t)} = Xw - \sum_{j=1}^{t} q_j X_j - \sum_{j=t+1}^{N} w^{(t)}_j X_j.
\end{equation}
In particular, we have $e_0 = 0$ before any quantization occurs, and $e_N = Xw - Xq$ once all coordinates have been quantized. To analyze how this error evolves through the OPTQ iterations \eqref{eq1} and \eqref{eq2}, we reformulate these updates in terms of least-squares problems. 
The following result, adapted from \cite{zhang2025qronos}, shows that OPTQ greedily minimizes  $e_t$ at each step by selecting the quantized value and then optimally adjusting the remaining coordinates.

\begin{lemma}[\cite{zhang2025qronos}]\label{thm:OPTQ2}
Lines \eqref{eq1} and \eqref{eq2} of OPTQ (\cref{OPTQ}) are equivalent to the pair of optimization problems:
\begin{align}
    q_t &= \underset{p \in \mathcal{A}}{\arg\min} \ \frac{1}{2} \left\| Xw - \sum_{j=1}^{t-1} q_j X_j - p X_t - \sum_{j=t+1}^{N} w^{(t-1)}_j X_j \right\|_2^2, \label{eq5} \\
    w^{(t)}_{\geq t+1} &= \underset{(v_{t+1}, \dots, v_N) \in \mathbb{R}^{N - t}}{\arg\min} \ \frac{1}{2} \left\| Xw - \sum_{j=1}^{t} q_j X_j - \sum_{j=t+1}^{N} v_j X_j \right\|_2^2. \label{eq6}
\end{align}
\end{lemma}

 Our first novel result is \cref{lemma:OPTQ_error_recur}, which is proved in \cref{appendix:OPTQ_error_recur}. It makes the error evolution explicit and expresses $e_t$ as a sum of projected quantization errors. Crucially, it also provides explicit OPTQ error bounds. 
\begin{proposition}[OPTQ Error Evolution and Bounds]\label{lemma:OPTQ_error_recur}\label{prop:OPTQ_deter_error}
Let $X \in \mathbb{R}^{m \times N}$ be full rank with $m \geq N$, and let $w \in \mathbb{R}^N$. Running OPTQ (\cref{OPTQ}) with $\lambda = 0$ (so $H = X^\top X$), the error defined in \eqref{eq:err_def} satisfies
\begin{align}
    e_t &= P_{X_{\geq t+1}^{\perp}}(w^{(t-1)}_t - q_t) X_t + e_{t-1} \text{\quad and \quad } e_N =\sum_{j=1}^{N} P_{X_{\geq j+1}^{\perp}}(w^{(j-1)}_j - q_j) X_j.\label{final error}
\end{align}
Moreover, the resulting quantized vector $q$ satisfies
\begin{gather}\label{eq:e_norm}
\|Xw - Xq\|_2^2 = \sum_{j=1}^N |w^{(j-1)}_j - q_j|^2 \, \|P_{X_{\geq j+1}^{\perp}} X_j\|_2^2. 
\end{gather}
In particular, this implies that when using the infinite alphabet $\mathcal{A}^\delta$
\begin{gather}\label{eq:e_norm_bound}
\|Xw - Xq\|_2 \leq \frac{\delta}{2} \sqrt{N} \cdot \min\left\{ \max_j \|P_{X_{\geq j+1}^\perp} X_j\|_2, \ \sqrt{\frac{\|X\|_F^2}{N}} \right\}.
\end{gather}
\end{proposition}
In the last portion of the above proposition, we assumed an {\emph{infinite}} quantization alphabet $\mathcal{A}^\delta = \{\pm k\delta : k \in \mathbb{Z}\}$ for simplicity, and we defer the discussion of finite alphabets for later.

The bounds above apply in the special case of unregularized OPTQ with a full-rank  matrix. Our next result extends this to arbitrary inputs $X \in \mathbb{R}^{m \times N}$ and includes a regularization parameter $\lambda > 0$. The resulting error bound introduces an explicit constant that quantifies the role of the conditioning of submatrices of $X$ and the effect of regularization.

\begin{theorem}[General $\ell_2$ Error Bound when $\lambda>0$]
\label{thm:l2 full version}
Let $X \in \mathbb{R}^{m \times N}$ and $w \in \mathbb{R}^N$. Running OPTQ (\cref{OPTQ}) with regularization parameter $\lambda > 0$ (so $H = X^\top X + \lambda I$) and alphabet $\mathcal{A}^\delta$, the resulting quantized vector $q$ satisfies
\begin{align}\label{l2 on sum}
    \|Xw - Xq\|_2^2 + \lambda \|w-q\|_2^2 
    \leq \dfrac{\delta^2}{4}N \cdot C_2(X,\lambda)^2. 
\end{align}
Consequently, we have
\begin{align}\label{l2 on both}
    \|Xw - Xq\|_2 \leq \frac{\sqrt{N} \delta}{2} \cdot C_2(X, \lambda), \quad  \text{and} \quad 
%
    \|w - q\|_2 \leq \frac{\sqrt{N} \delta}{2} \cdot \frac{C_2(X, \lambda)}{\sqrt{\lambda}},
\quad \text{where} \quad \ &\\
\label{eq:C}
    C_2(X, \lambda)^2 := \min\left\{
        \max\left\{
            \max_{j \leq N - m} \frac{\lambda \|X_j\|_2^2}{(\sigma^{(j)}_{\min})^2 + \lambda}, \ 
            \max_{j > N - m} \|X_j\|_2^2
        \right\}, \ 
        \frac{\|X\|_F^2}{N}
    \right\} &+ \lambda,
\end{align}
and $\sigma^{(j)}_{\min}$ denotes the smallest non-zero singular value of $X_{\geq j+1}$. When $N \leq m$, the index set $\{j \leq N - m\}$ is empty and the corresponding term is omitted.
\end{theorem}

\begin{proof}
    As  OPTQ with $\lambda>0$ is equivalent to OPTQ without dampening, applied to $\widehat{X}=\begin{pmatrix}
    X \\ 
    \sqrt{\lambda} I
\end{pmatrix}$, then by \cref{prop:OPTQ_deter_error}
\begin{gather*}
  \|\hat{X}w - \hat{X}q\|_2\leq \dfrac{\delta}{2}\sqrt{N}\min \left\{ \max_{j}\|P_{\hat{X}_{\geq j+1}^{\perp}}\hat{X}_{j}\|_2, \sqrt{\frac{\|X\|_F^2}{N} + \lambda}  \right\}.
\end{gather*}
Moreover, by \cref{lemma:projection_upper_bound}, one can further deduce
\begin{align*}
      \|P_{\widehat{X}^\perp_{\geq j+1}}\widehat{X}_{j}\|_2^2 \leq \begin{cases}
          \frac{\lambda}{(\sigma^{(j)}_{\min})^2+\lambda} \cdot \|X_{j}\|_2^2 + \lambda&\text{ when $j\leq N-m$}\\
          \|X_j\|_2^2 + \lambda&\text{ when $j > N-m$}
      \end{cases}.
\end{align*}
This implies
\begin{gather*}
  \max_{j}\|P_{\widehat{X}^\perp_{\geq j+1}}\widehat{X}_{j}\|_2\leq   \max\left\{\max_{j \leq N-m} \ \frac{\lambda \|X_j\|_2^2}{(\sigma^{(j)}_{\min})^2+\lambda} , \ \max_{j >N-m}\|X_{j}\|_2^2 \right\} + \lambda.
\end{gather*}
Thus \cref{l2 on sum} follows,
\begin{align*}
    \|Xw - Xq\|_2^2 + \lambda \|w-q\|_2^2 &= \|\widehat{X}w - \widehat{X}q\|_2^2
    \leq \dfrac{\delta^2}{4}N \cdot C_2(X,\lambda)^2. 
\end{align*}
Since $ \|Xw - Xq\|_2^2$ and $\lambda \|w-q\|_2^2$ are each bounded by $\|Xw - Xq\|_2^2 + \lambda \|w-q\|_2^2$, we obtain the desired bounds in \cref{l2 on both}. 
\end{proof}

\subsection{Insights and Practical Implications}\label{sec:practical}
We now explore the practical implications of our theoretical results. We show how they help explain several design choices commonly made in OPTQ implementations, including column ordering, the choice of $\lambda$, as well the effectiveness of OPTQ relative to simpler quantization baselines.

\begin{remark}[A \new{Heuristic} Justification for Decreasing‐Norm Ordering]\label{rem:decreasing}
    \Cref{eq:C} allows a \new{heuristic} explanation for the widely used \new{practical strategy} of sorting the columns of $X$ in decreasing $\ell_2$ norm order \cite{frantar2022gptq, zhang2025qronos}. 
    \new{We first consider the case when \(N>m\) and \(X\in\R^{m\times N}\) is in general position\footnote{That is, every subset of \(m\) columns is linearly independent.}. The upper bound is governed by the larger of
\[
    \max_{j \le N - m} \frac{\lambda \|X_j\|_2^2}{(\sigma_{\min}^{(j)})^2 + \lambda}
    \quad \text{and} \quad
    \max_{j > N - m} \|X_j\|_2^2 .
\]
Accordingly, one seeks to keep both quantities small. The second term is minimized by placing the \(m\) columns of smallest norm at the end of the sequence \((X_j)_{j=1}^N\). For the first term, \(\sigma_{\min}^{(j)}\) is non-increasing in \(j\) for all \(j \le N-m\), independent of the column ordering, see \Cref{lem:gen_pos}. Hence, the factor
\(
\frac{\lambda}{(\sigma_{\min}^{(j)})^2 + \lambda}
\)
is non-decreasing in \(j\). If \(\|X_j\|_2^2\) is also increasing for \(j \le N-m\), then the first term is driven by the product of two increasing sequences. Ordering the columns \(X_j\), \(j \le N-m\), by decreasing norm instead pairs this increasing factor with a decreasing one, thereby helping control the bound.}
    
\new{We now examine the case when \(m>N\). By \cref{eq:e_norm}, as \(j\) increases, the subspace \(X_{\geq j+1}^{\perp}\) becomes larger, making the corresponding projection less likely to absorb error. It is therefore natural to order the columns \(X_j\) by decreasing norm, to pair larger vectors with projections onto smaller subspaces. This perspective is further supported when \(X \approx UR\) is approximately low-rank, where \(U \in \mathbb{R}^{m \times r}\) has orthonormal columns, \(R \in \mathbb{R}^{r \times N}\), and \(r \ll \min(m,N)\). Here, the geometry is essentially governed by the wide matrix \(R\), regardless of the shape of \(X\). We rigorously justify descending norm ordering in the exact low-rank setting in \cref{cor:low rank X}.}

    
\end{remark}
The following corollary will help us both compare OPTQ to MSQ, and better understand the role of $\lambda$.

\begin{corollary}
\label{thm:l2 simp ver}
Let $X \in \mathbb{R}^{m \times N}$ and $w \in \mathbb{R}^N$. {When} running OPTQ (\cref{OPTQ}) with regularization parameter $\lambda > 0$ (so $H = X^\top X + \lambda I$) and alphabet $\mathcal{A}^\delta$, the resulting quantized vector $q$ satisfies
    \begin{equation}\label{l2 control on Xw}
    \|{X}w-{X}q\|_2\leq \dfrac{\sqrt{N}\delta}{2}\min\left\{\sqrt{\dfrac{\Tr(X^TX)}{N}+\lambda} \ , \|X\|_{\mathrm{op}}\right\}
\end{equation}
and
\begin{equation}\label{l2 control on w}
    \|w-q\|_2\leq \dfrac{\sqrt{N}\delta}{2}\sqrt{\dfrac{\Tr(X^TX)}{N\lambda}+1},
\end{equation}
\end{corollary}

\begin{proof} 
Applying the inequalities $(\sigma^{(j)}_{\min})^2+\lambda \geq \lambda$ and $\frac{\|X\|_F^2}{N}\leq \max_{j}\|X_{j}\|_2^2$ to  \cref{eq:C} yields  $C_2(X,\lambda)^2 \leq \frac{\|X\|^2_F}{N}+\lambda$. Using \cref{l2 on both}, we immediately have $\|{X}w-{X}q\|_2\leq \dfrac{\sqrt{N}\delta}{2}\sqrt{\dfrac{\Tr(X^TX)}{N}+\lambda}$ and $\|w-q\|_2\leq \dfrac{\sqrt{N}\delta}{2}\sqrt{\dfrac{\Tr(X^TX)}{N\lambda}+1}$ which proves \cref{l2 control on w} and half of \cref{l2 control on Xw}. 

To finish the proof of \cref{thm:l2 simp ver}, it remains to be shown that \(
\|Xw - Xq\|^2 \leq \frac{N \delta^2}{4} \cdot \|X\|^2_{\mathrm{op}}
\). From \cref{l2 on sum}, we can derive $\|Xw-Xq\|^2+\lambda \|w-q\|^2\leq \frac{\delta^2}{4}N(\frac{\|X\|^2_F}{N}+\lambda)$. Equivalently, $\|Xw-Xq\|^2\leq \frac{\delta^2\|X\|^2_F}{4}+\lambda (\frac{N\delta^2}{4}-\|w-q\|^2)$. When $\|w-q\|^2\geq \frac{N\delta^2}{4}$, we have $\|Xw-Xq\|^2\leq \frac{\delta^2\|X\|^2_F}{4}\leq \frac{N \delta^2}{4} \cdot \|X\|^2_{\mathrm{op}}$. When $\|w-q\|^2\leq \frac{N\delta^2}{4}$, we have the direct operator norm bound \(
\|Xw - Xq\|^2 \leq \frac{N \delta^2}{4} \cdot \|X\|^2_{\mathrm{op}}
\). Thus, in both cases, we have \(
\|Xw - Xq\|^2 \leq \frac{N \delta^2}{4} \cdot \|X\|^2_{\mathrm{op}}
\). As we already showed $\|{X}w-{X}q\|_2\leq \dfrac{\sqrt{N}\delta}{2}\sqrt{\dfrac{\Tr(X^TX)}{N}+\lambda}$, we conclude that \cref{l2 control on Xw} 
holds.
\end{proof}

\begin{remark}[Comparison to MSQ-style Bounds]
The bound in \cref{thm:l2 simp ver} shows how OPTQ improves upon memoryless scalar quantization (MSQ) applied to each coordinate of $w$ independently. Specifically, MSQ gives the uniform bound
\[
\|Xw - Xq\|_2 \leq \frac{\sqrt{N} \delta}{2} \cdot \|X\|_{\mathrm{op}},
\]
where $\|X\|_{\mathrm{op}}$ denotes the spectral norm. Since $\|X\|_{\mathrm{op}}\geq \max\|X_j\| \geq 
\sqrt{\Tr(X^\top X)/N}$, our result shows that OPTQ replaces the worst-case operator norm with a  smaller quantity.

To quantify the potential improvement, consider a matrix $X$ whose columns are all identical with norm $\|X_i\|_2 = \sqrt{m}$. Then $\|X\|_F = \|X\|_{\mathrm{op}} = \sqrt{mN}$, giving the MSQ bound $\|Xw - Xq\|_2 = O(\sqrt{m}N)$. In contrast, our OPTQ bound in \cref{thm:l2 simp ver}, with a small $\lambda$, is $O(\sqrt{mN})$. 
More generally, in \cref{l2 control on Xw}, one generically  expects a gap between $\frac{\Tr(X^TX)}{N}$ and the larger quantity $\|X\|^2_{\mathrm{op}}$. So when $\lambda$ is small, we have $\dfrac{\sqrt{N}\delta}{2}\cdot \sqrt{\frac{\Tr(X^TX)}{N}+\lambda}$ as the OPTQ error bound. As  $\lambda$ increases, the OPTQ bound becomes $\dfrac{\sqrt{N}\delta}{2}\cdot\|X\|_{\mathrm{op}}$ so that as $\lambda \rightarrow \infty$, \cref{l2 control on Xw} reduces to $\|Xw-Xq\|_2 \leq \dfrac{\sqrt{N}\delta}{2}\cdot\|X\|_{\mathrm{op}}$ and \cref{l2 control on w} reduces to $\|w-q\|_2\leq \dfrac{\sqrt{N}\delta}{2}$ which are the MSQ bounds. This corresponds to the fact that $H = X^\top X + \lambda I$ is essentially a scaled identity matrix as $\lambda \rightarrow \infty$ and running OPTQ in that case is equivalent to using MSQ.

\end{remark}

\begin{remark}\label{remark:lamda_suggestion}[Choice of $\lambda$, alphabet size, and the need for $\ell_\infty$ bounds]\label{remark:finte alphabet l2}
\Cref{thm:l2 simp ver} heuristically justifies choosing $\lambda$ as a small constant multiple of $\|X\|_F^2/N$. This aligns with the recommendation in \cite{frantar2022gptq}, where $\lambda$ is set to $0.01 \cdot \|X\|_F^2/N$. With this choice, the bound
\(
\|w - q\|_2 \leq O(\sqrt{N}) \delta
\)
implies that $q$ deviates from $w$ by approximately $O(1) \delta$ per entry \emph{on average}. If $\|w-q\|_\infty \leq O(1) \delta$ ---as one might expect generically--- then a finite alphabet of the form
\[
\mathcal{A}_b^{\delta} = \left\{ \pm k \delta : k \in \{-2^{b-1}, \dots, -1, 0, 1, \dots, 2^{b-1} \} \right\}
\]
suffices, provided $2^{b-1} \delta \geq \|w\|_\infty + O(1)\delta$. The additive $O(1)\delta$ term accounts for the price of adaptive rounding: additional dynamic range is needed to absorb errors that arise from projection-based cancellation.
While this heuristic is likely valid in most practical instances where OPTQ is applied, it cannot be made fully rigorous. In particular, there exist matrices \(X\) and vectors \(w\) for which \Cref{thm:l2 simp ver} yields the \new{pessimistic $\ell_\infty$} upper bound
\(
\|q\|_\infty \leq \|w\|_\infty + O(\sqrt{N})\,\delta,
\)
\new{and this bound is sharp even for ``reasonably" chosen  \(\lambda\), in the sense that one can construct examples with
\[
\|q\|_\infty = \|w\|_\infty + O(\sqrt{N})\,\delta, \qquad \|X(w-q)\|_\infty = O(\sqrt{N})\delta,
\]
despite $\|X\|_{\mathrm{op}} = O(1)$ and $\|w\|_\infty\le \frac12$. We construct such an example in \cref{adv lambda greater 0}. Combined with \Cref{thm:l2 simp ver}, this further shows that the $\ell_2$ bound in \cref{l2 control on w} is tight. We also give a second example in \cref{adv lambda eq 0} showing that, when \(\lambda = 0\), an even worse case can arise, with
\(
\|w-q\|_\infty = O(N)\,\delta
\) and \(
\|Xw-Xq\|_\infty = O(\sqrt{N})\,\delta,
\)
}
In \cref{sec:OPTQ linf error bound}, we improve the  dependence on $N$ in the upper bound controlling $\|q\|_\infty$ from $\sqrt{N}$ to $\sqrt{\log N}$ by using a stochastic RTN operator $\mathcal{Q}_{\text{stoc}}$. See \cref{remark:finte alphabet}.
\end{remark}

\begin{remark}[Generalization]\label{remark:generalization}
  \Cref{thm:l2 simp ver} also sheds light on how regularization may help generalization. Consider a single neuron represented by a weight vector $w \in \mathbb{R}^N$, and let $X \in \mathbb{R}^{m \times N}$ denote the calibration dataset. Suppose $q_X \in \mathcal{A}^N$ is the quantized version of $w$ obtained using $X$, where $\mathcal{A}$ denotes the quantization alphabet. Then, for an unseen random data point $z \in \mathbb{R}^N$, we have
  \begin{align*}
        &\mathbb{E}_{z} |z^{\top}(w-q_X)|^2
        =\mathbb{E}_{z}(w-q_X)^{\top} zz^{\top} (w-q_X)\\
        &=\mathbb{E}_{z} (w - q_X)^{\top} \cdot \dfrac{1}{m}X^{\top}X\cdot (w-q_X) + \mathbb{E}_{z}(w-q_X)^\top (zz^\top - \dfrac{1}{m}X^\top X)(w-q_X)\\
        &=\dfrac{1}{m} \|X(w-q_X)\|_2^2 + (w-q_X)^\top (\mathbb{E}_{z}[zz^\top] - \tfrac{1}{m}X^\top X)(w-q_X).
    \end{align*}
    This decomposition provides a sufficient condition for achieving low generalization error. First, the generalization error depends on the reconstruction error over the calibration set $X$. As \cref{thm:l2 simp ver} shows, this term can be effectively controlled for a well-designed quantization algorithm such as OPTQ. Second, the generalization error depends on the proximity between the original weight vector $w$ and its quantized version $q_X$. This proximity can be enforced through regularization, as demonstrated in \cref{thm:l2 simp ver} and \cref{remark:finte alphabet l2}. Third, it depends on the quality of the empirical estimate of the second-moment matrix $\mathbb{E}_{z}[zz^\top]$ by the empirical average $\tfrac{1}{m} X^\top X = \tfrac{1}{m} \sum_{i=1}^m x_i x_i^\top$, where $x_i$ denotes the $i$-th row of $X$. This shows it is important for the calibration dataset to be representative of the underlying data distribution.
\end{remark}

\section{$\ell_\infty$-Norm Error Analysis of OPTQ with Stochastic Rounding}
\label{sec:OPTQ linf error bound}
The results in \Cref{sec:OPTQ l2 error bound} are the first quantitative error bounds for a deterministic PTQ algorithm. However, they apply only to the $\ell_2$ norm of the error, and are not fine enough to handle  entry-wise control of $Xw-Xq$, which would be desirable for a number of reasons. 

First, one important difficulty in analyzing OPTQ is the lack of direct control on the magnitude of entries $\|w^{(t)}\|_\infty$ when they are being updated in iterations,  which makes it difficult to bound the quantization grid-size for a given number of bits, or alternatively the number of bits needed for quantization. Although we have already derived a bound on $\|w-q\|_2$, it unfortunately still scales with $\sqrt{N}$ and is not fine enough for this purpose. Developing a technique that enables controlling the $\ell_{\infty}$ norm error would resolve this issue, as we will see in \cref{sec:OPTQ linf error bound}. To achieve that, we adopt a different approach that replaces the deterministic RTN operator $\Q$ used in \cref{OPTQ} by an unbiased stochastic rounding operator $\mathcal{Q}_{stoc}$ as in \cite{zhang2023spfq}. 

Second, $Xq$ is the pre-activation feeding into the next layer of the network, and as such will need to itself (after a non-linearity) be entry-wise quantized in an activation quantization setting. Guaranteeing a small $\|Xw-Xq\|_\infty$ would therefore enable quantizing these activations with a reasonable grid-size. 

Third, in neural networks, one often encounters important non-linearities like ``Softmax", $\sigma(z)_i:=\exp(z_i)/\sum_j\exp(z_j)$, which act on logits $Xw$ (or $Xq$), turning them into a probability vector where the largest coordinates are the most important (e.g., for classification, or next-token prediction). 
{Moreover, in the context of modern large language models (LLMs), the top-$k$ logits are often the only information used by the latest search-based decoding algorithms \cite{shi2024thorough}. 
As such, preserving the $\ell_\infty$ norm ensures that the most probable tokens are reliably identified even if exact probabilities are not preserved.}
{When quantizing with OPTQ, there is a danger that a single large logit error can flip the ranking.} An $\ell_\infty$ bound controls every coordinate, so if $\|Xw-Xq\|_\infty$ does not exceed half the entrywise gap within the sorted entries of $Xw$ , the ordering—and thus the output—remains intact. An $\ell_2$ bound cannot guarantee this, as it may appear small yet still hide a large coordinate spike. 

{Unfortunately, OPTQ can result in an error with $\|Xw-Xq\|_\infty$ scaling as $\sqrt{N}$, much too large to address the second and third points above. 
To make this claim concrete, we provide examples of such $X$ and $w$ for which OPTQ results in $\|Xw - Xq\|_\infty = \|Xw - Xq\|_2 = O(\sqrt{N})$, in \cref{Appendix:adversarial}.}

\subsection{Entry-wise Error Bounds for OPTQ with Stochastic Rounding}
To establish more favorable $\ell_{\infty}$ norm error bounds, 
we consider a modified version of \cref{OPTQ} in which the original deterministic quantizer $\mathcal{Q}$ (appearing in \eqref{eq1}, \eqref{eq3}, and \eqref{eq5}) is replaced with the unbiased stochastic quantizer $\mathcal{Q}_{stoc}$ defined in \cref{sec:notation}. To analyze this stochastic variant of OPTQ, we build on techniques from \cite{zhang2023spfq, alweiss2021discrepancy}, together with \cref{lemma:OPTQ_error_recur}. For simplicity, we initially assume an {\emph{infinite}} quantization alphabet $\mathcal{A} = \{\pm k\delta : k \in \mathbb{Z}\}$, then show how this assumption can be removed.

As before, we begin with $\lambda=0$. Our goal is to analyze the quantization error when applying \cref{OPTQ} with $\mathcal{Q}_{stoc}$ to a single layer with layer input $X$. Let \( W \in \mathbb{R}^{N \times N'} \) denote the weight matrix of the layer and \( Q \in \mathcal{A}^{N \times N'} \) the output weight matrix {quantized by} the algorithm. We are interested in controlling the entry-wise $\ell_\infty$ error, \( \max_{i,j} |(XW - XQ)_{ij}| \), with high probability. Since each neuron is quantized in parallel, we can study each neuron independently and derive a whole layer error result from that. We need the following important definition of \emph{convex ordering} whose properties we summarize in \cref{sec:cvx ord prop}.
\begin{definition}[Convex Order]
	Let $ X, Y $ be $ n$-dimensional random vectors such that
	\begin{gather*}
		\mathbb{E}f(X)\leq\mathbb{E}f(Y)
	\end{gather*}
	holds for all convex functions $ f:\mathbb{R}^{n}\rightarrow \mathbb{R} $, provided the expectations exist. Then $ X $ is said to be smaller than $ Y $ in the \textit{convex order}, denoted by $ X\prec_{cx}Y $.
\end{definition}

In view of the properties in \cref{sec:cvx ord prop}, particularly \eqref{lemma:cx-bounded}, it is natural to bound the final error $e_N$ by a Gaussian (in the sense  of convex ordering) as that will allow us to control the entry-wise magnitude of $Xw-Xq$ with high probability. The next lemma, which we prove in 
\cref{proof:bound_one_neuron_OPTQ}, 
provides this Gaussian upper bound and controls its associated covariance.

\begin{lemma}[Convex Order Dominance of the Error]\label{lemma:bound_one_neuron_OPTQ}~
Let $q$ be the output of quantizing $w$ with OPTQ with stochastic quantizer $\mathcal{Q}_{stoc}$, then $ Xw - Xq \prec_{cx} \mathcal{N}(0, \Sigma) $, where \begin{align*}
 	 \Sigma&=\dfrac{\pi\delta^{2}}{2}\sum_{j=1}^{N}P_{X_{\geq j+1}^{\perp}}X_{j}X_{j}^{\top} P_{X_{\geq j+1}^{\perp}}\\&\preceq \dfrac{\pi \delta^{2}}{2}\max_{j}\|P_{X^\perp_{\geq j+1}}X_{j}\|_2^{2} I.
 \end{align*}
\end{lemma}

This now allows us to obtain an entry-wise $\ell_\infty$-norm upper bound on the reconstruction error $XW-XQ$. 

\begin{theorem}\label{thm:OPTQ error bound}
Let $X \in \mathbb{R}^{m \times N}$ be full rank with $m \geq N$, and let $W \in \mathbb{R}^{N \times N'}$. Run OPTQ (\cref{OPTQ}) with stochastic rounding operator $\mathcal{Q}_{\mathrm{stoc}}$, infinite alphabet $\mathcal{A}^\delta$, and $\lambda = 0$ (so $H = X^\top X$). Then for any $p,p'>0$
    and any column $w$ of $W$ with quantized version $q$, we have
    \[
    \| Xw - Xq \|_\infty 
    \leq \delta \sqrt{2 \pi p \log N} \cdot \max_j \| P_{X_{\geq j+1}^\perp} X_j \|_2
    \]
    with probability at least
    \(
    1 - \frac{ \sqrt{2} m }{ N^p }.
    \)
         Moreover, for the full matrix $W$, the quantized matrix $Q$ satisfies
    \[
    \max_{i,j} |(XW - XQ)_{ij}| 
    \leq \delta \sqrt{2 \pi (p \log N + p' \log N')} \cdot \max_j \| P_{X_{\geq j+1}^\perp} X_j \|_2
    \]
    with probability at least
    \(
    1 - \frac{ \sqrt{2} m }{ N^p {N'}^{p'-1} }.
    \)
\end{theorem}
\begin{proof}
	For one neuron (column) $ w $ of $ W $, combining \cref{lemma:bound_one_neuron_OPTQ}, 
    and \cref{lemma:convex_order} \cref{trans},  we have
	\begin{gather*}
		Xw - Xq \prec_{cx} \mathcal{N}\left(0, \dfrac{\pi \delta^{2}}{2}\max_{j}\|P_{X^\perp_{\geq j+1}}X_{j}\|_2^{2}I\right).
	\end{gather*}

Then applying  \cref{lemma:convex_order} \cref{lemma:cx-gaussian-tail} with $ \alpha=\delta\sqrt{2\pi (p \log{N}+ p' \log{N'})}\max_{j}\|P_{X^\perp_{\geq j+1}}X_{j}\|_2 $ and taking a union bound over all neurons completes the proof. In particular, one may simply set $N'=1$ to obtain the single neuron result.  
\end{proof}  

\begin{remark}[Interpretation of the Success Rate]
    For the success rate on the full layer $W$ to be at least $1-\epsilon$, we can set $p, p' >0$ such that $\frac{ \sqrt{2} m }{ N^p {N'}^{p'-1} } = \epsilon$, which is equivalent to 
    $p \log{N}+ (p'-1) \log{N'}=\log \frac{\sqrt{2} m}{\epsilon}$. Then, the quantized matrix $Q$ satisfies  
    \[
    \max_{i,j} |(XW - XQ)_{ij}| 
    \leq \delta \sqrt{2 \pi \log \frac{\sqrt{2} m N'}{\epsilon}} \cdot \max_j \| P_{X_{\geq j+1}^\perp} X_j \|_2
    \]
    with probability at least $1-\epsilon$. One can similarly interpret the success rate for all remaining results in this section, so we will not repeat this calculation.
\end{remark}

\begin{remark}[Near-Optimality of the Upper Bound]
In the bound for a column $w$,  we have \[
    \| Xw - Xq \|_\infty 
    \leq \delta \sqrt{2 \pi p \log N} \cdot \max_j \| P_{X_{\geq j+1}^\perp} X_j \|_2.
    \]
    Taking $w=0, \delta=1$ and assuming each $\|X_j\|_2 \leq 1$ reduces the quantization problem into a vector balancing problem and our bound becomes $\| Xq \|_\infty 
    \lesssim \sqrt{\log N}$. The vector balancing problem is the subject of the Komlós conjecture (see, e.g., \cite{alweiss2021discrepancy}) and the best known bound is $\mathcal{O}(\sqrt{\log \min\{m, N\}})$ when the alphabet is binary, i.e., when $\mathcal{A}=\{-1, 1\}$.
\end{remark}

\begin{theorem}[General $\ell_\infty$ Error Bound when $\lambda>0$]\label{thm:l inf full ver}
Let $X \in \mathbb{R}^{m \times N}$, let $W \in \mathbb{R}^{N \times N'}$, and let $\lambda > 0$. Run OPTQ (\cref{OPTQ}) with stochastic rounding operator $\mathcal{Q}_{\mathrm{stoc}}$, infinite alphabet $\mathcal{A}^\delta$, and $H = X^\top X + \lambda I$. Then for any $p, p' > 0$ and any column $w$ of $W$ with quantized version $q$,
\[
\| Xw - Xq \|_\infty 
\leq \delta \sqrt{2 \pi p \log N} \cdot C_\infty(X, \lambda)
\quad \text{and} \quad
\| w - q \|_\infty 
\leq \delta \sqrt{2 \pi p \log N} \cdot \frac{ C_\infty(X, \lambda) }{ \sqrt{\lambda} }
\]
with probability at least 
\( 1 - \frac{ \sqrt{2} (m + N) }{ N^p } \).
Moreover, with probability at least 
\( 1 - \frac{ \sqrt{2} (m + N) }{ N^p {N'}^{p' - 1} } \),
for the full matrix $W$ the quantized matrix $Q$ satisfies
\begin{align*}
\max_{i,j} |(XW - XQ)_{ij}|
&\leq \delta \sqrt{2 \pi (p \log N + p' \log N')} \cdot C_\infty(X, \lambda)
\quad \text{and} \quad \\
\max_{i,j} |(W - Q)_{ij}|
&\leq \delta \sqrt{2 \pi (p \log N + p' \log N')} \cdot \frac{ C_\infty(X, \lambda) }{ \sqrt{\lambda} },\quad 
\end{align*}
where
\( C_\infty(X, \lambda)^2 = 
\max \left\{
    \max_{j \leq N - m} 
        \frac{ \lambda \| X_j \|_2^2 }{ ( \sigma_{\min}^{(j)} )^2 + \lambda }, \ 
    \max_{j > N - m} 
        \| X_j \|_2^2
\right\}
+ \lambda,
\)
and  $\sigma_{\min}^{(j)}$ denotes the smallest singular value of $X_{\geq j+1}$.
\end{theorem}
\begin{proof}
    As discussed before, running \cref{OPTQ} with $\lambda >0$ using data matrix $X$ is equivalent to running \cref{OPTQ} with $\lambda=0$ using data matrix $\widehat{X}=\begin{pmatrix}
    X \\ 
    \sqrt{\lambda} I
\end{pmatrix}$. Then we can use \cref{thm:OPTQ error bound} to deduce that  
\begin{gather}\label{general_pos_bound}
    \max_{i,j}|(\widehat{X}W - \widehat{X}Q)_{ij}|\leq \delta \sqrt{2\pi (p \log{N}+ p' \log{N'})}\max_{j}\|P_{\widehat{X}^\perp_{\geq j+1}}\widehat{X}_{j}\|_2
\end{gather}
with probability greater than $ 1-\tfrac{\sqrt{2}(m+N)}{N^{p}{N'}^{(p'-1)}}$. It suffices to bound $\max_{j}\|P_{\widehat{X}^\perp_{\geq j+1}}\widehat{X}_{j}\|_2$. From \cref{lemma:projection_upper_bound}, we know
\begin{align*}
      \|P_{\widehat{X}^\perp_{\geq j+1}}\widehat{X}_{j}\|_2^2 \leq \begin{cases}
          \frac{\lambda}{(\sigma^{(j)}_{\min})^2+\lambda} \cdot \|X_{j}\|_2^2 + \lambda&\text{ when $m\leq N-j$}\\
          \|X_j\|_2^2 + \lambda&\text{ when $m > N-j$}
      \end{cases}.
\end{align*}
Then we have
\begin{gather*}
  \max_{j}\|P_{\widehat{X}^\perp_{\geq j+1}}\widehat{X}_{j}\|_2\leq   \max\left\{\max_{j \leq N-m} \ \frac{\lambda \|X_j\|_2^2}{(\sigma^{(j)}_{\min})^2+\lambda} , \ \max_{j >N-m}\|X_{j}\|_2^2 \right\} + \lambda.
\end{gather*}
Combining the above inequality with \cref{general_pos_bound}, we obtain
\begin{align*}
    \max_{i,j}|(\widehat{X}W - \widehat{X}Q)_{ij}|\leq  
           \delta & \sqrt{2\pi  (p \log{N}  + p' \log{N'})} \\ &\times \sqrt{\max\left\{\max_{j \leq N-m} \ \frac{\lambda \|X_j\|_2^2}{(\sigma^{(j)}_{\min})^2+\lambda} , \ \max_{j >N-m}\|X_{j}\|_2^2 \right\}+ \lambda}.
\end{align*}
Then use the fact that \begin{gather*}
    \widehat{X}W - \widehat{X}Q=\begin{pmatrix}
    XW-XQ \\ 
    \sqrt{\lambda}(W-Q)
\end{pmatrix}.
\end{gather*} to finish the proof, setting $N'=1$ to obtain the single-column result.
\end{proof}

\subsection{Insights and Practical Implications}\label{sec:practical2}

We now present a corollary that allows us to explore the practical implications of our theoretical results. In particular, we will derive insights into the size of the alphabet needed for OPTQ, as well as into the role of the rank of $X$.

\begin{corollary}\label{cor:finite alphabet}
Let $X \in \mathbb{R}^{m \times N}$, let $W \in \mathbb{R}^{N \times N'}$, and let $\lambda > 0$. Run OPTQ (\cref{OPTQ}) with stochastic rounding operator $\mathcal{Q}_{\mathrm{stoc}}$, infinite alphabet $\mathcal{A}^\delta$, and $H = X^\top X + \lambda I$. Then for any $p, p' > 0$, the quantized matrix $Q$ satisfies
\begin{align*}
\max_{i,j} |(XW - XQ)_{ij}|
&\leq \delta \sqrt{2 \pi (p \log N + p' \log N')} \cdot \sqrt{ \max_j \| X_j \|_2^2 + \lambda } 
\\
\text{and} \quad \max_{i,j} |(W - Q)_{ij}|
&\leq \delta \sqrt{2 \pi (p \log N + p' \log N')} \cdot \sqrt{ \max_j \frac{ \| X_j \|_2^2 }{ \lambda } + 1 }
\end{align*}
with probability at least 
\( 1 - \frac{ \sqrt{2} (m + N) }{ N^p {N'}^{p' - 1} } \).
\end{corollary}

The proof of \cref{cor:finite alphabet} simply follows from the fact that in \cref{thm:l inf full ver}, one has $(\sigma_{\min}^{(j)})^2 + \lambda \geq \lambda$.

\begin{remark}[A small finite alphabet suffices]\label{remark:finte alphabet}
Unlike the setting in \cref{remark:finte alphabet l2}, where a larger alphabet is required due to a $\sqrt{N}$-scale additive term in the dynamic range of $q$, \cref{cor:finite alphabet} shows that stochastic rounding refines this dependence to $\sqrt{\log N}$. 
Consider a single neuron $w$ (i.e., $N' = 1$) as in \cref{remark:finte alphabet l2}, and let $q$ be its quantized counterpart. If the regularization parameter $\lambda$ is chosen on the scale of $\max_i \|X_i\|_2^2$, then \cref{cor:finite alphabet} implies
\(
    \|q\|_\infty \leq \|w\|_\infty + \mathcal{O}(\sqrt{\log N}) \cdot \delta.
\)
Now suppose we quantize using the symmetric finite alphabet
\(
\mathcal{A}^\delta_b = \left\{ \pm k\delta : k \in \{-2^{b-1}, \dots, -1, 0, 1, \dots, 2^{b-1} \} \right\}.
\)
It suffices to ensure that
\[
2^{b-1} \delta \geq \|w\|_\infty + \mathcal{O}(\sqrt{\log N}) \cdot \delta
\]
so that all quantized values $q$ fall within this finite alphabet. The additional range needed to accommodate adaptive rounding thus scales only with $\sqrt{\log N}$—a substantial improvement over the deterministic setting, where the required expansion can be as large as $\mathcal{O}(\sqrt{N})$ (see \cref{adv lambda greater 0} for an example).
To quantify the bit savings, define $K := \|w\|_\infty / \delta$. Then the number of bits that covers the dynamic range is 
\(
\mathcal{O}\left(\log(K + \sqrt{\log N})\right)\) in the stochastic case versus 
\(
\mathcal{O}\left(\log(K + \sqrt{N})\right)
\)
in the deterministic case. 
{The relative gap between these terms increases as $K$ decreases, and is more significant in the low-bit regime.}
\end{remark}

We have already noted that the Cholesky and least-squares formulations of OPTQ are equivalent when $H = X^\top X$ is invertible. Moreover, the Cholesky formulation is usually more computationally favorable when it exists. In practice, pre-trained model weights and activations often exhibit approximate low-rank structure \cite{huh2021low, zhang2025theoretical, zhang2024magr, zhang2024qera}. This makes it necessary to add $\lambda > 0$ so that $H = X^\top X + \lambda I$ is well-conditioned for Cholesky-based OPTQ, though choosing $\lambda$ can itself be non-trivial \cite{zhang2025qronos, egiazarian2024extreme}. The least-squares formulation (as in \cref{thm:OPTQ1}), in contrast, can be applied even if $X$ is low-rank or if $m<N$. 
The next corollary shows that the least-squares implementation yields a tighter error bound when $X$ is low-rank.
\begin{corollary}[Low-rank $X$]\label{cor:low rank X}
Let $X = UR$ where $U \in \mathbb{R}^{m \times r}$ has orthonormal columns, $R \in \mathbb{R}^{r \times N}$, and $r \ll \min(m, N)$. Assume $R$ is in general position, i.e., any $r$ columns of $R$ are linearly independent. Run OPTQ using the least-squares formulation (\cref{thm:OPTQ1}) with stochastic quantizer $\mathcal{Q}_{\mathrm{stoc}}$ and the infinite alphabet $\mathcal{A}^\delta$. Then for any $p, p' > 0$,
\[
\max_{i,j} |(XW - XQ)_{ij}| 
\leq \delta \sqrt{2 \pi (p \log N + p' \log N')} \cdot \max_{j > N - r} \| X_j \|_2
\]
with probability at least 
\( 1 - \frac{ \sqrt{2} r }{ N^p {N'}^{p' - 1} } \).
\end{corollary}
\begin{proof}
As before, it suffices to study a column (neuron) $w$ and the result for a layer $W$ will follow from \cref{lemma:convex_order} \cref{lemma:cx-gaussian-tail} and a union bound over all neurons. 
By the low-rank assumption and \cref{lemma:OPTQ_error_recur}, $e_{N-r}=\textbf{0}=Xw-Xw^{(N-r)}$, with \(
w^{(N-r)} = (q_{\leq N-r}, w^{(t-1)}_{\geq N-r+1})
\). Combining \cref{lemma:bound_one_neuron_OPTQ} 
and \cref{lemma:convex_order} \cref{trans},  we have
	\begin{gather*}
		Xw - Xq = Xw^{(N-r)} - Xq = Xw^{(N-r)} - Xw^{(N)} \prec_{cx} \mathcal{N}\left(0, \dfrac{\pi \delta^{2}}{2}\max_{j>N-r}\|X_{j}\|_2^{2}I\right).
	\end{gather*}
Then applying \cref{lemma:convex_order} \cref{lemma:cx-gaussian-tail} with $ \alpha=\delta\sqrt{2\pi (p \log{N}+p'\log{N'})} \underset{j>N-r}{\max}\|X_{j}\|_2 $ and  a union bound over all neurons completes the proof. 
\end{proof}

\begin{remark}[Further support for  norm ordering]    The above corollary helps further justify the common practice of reordering the columns of $X$ in descending order of $\|X_j\|_2$ \cite{brevitas, ist2022gptq} as $\underset{j>N-r}{\max}\|X_{j}\|_2$ may be significantly smaller than $\max\limits_{j\in[N]}\|X_{j}\|$ when the columns are sorted. \end{remark}

\section{An Error Analysis of Qronos}\label{sec:Qronos theory}
Using similar techniques, we analyze Qronos (\cref{Qronos++}) in this section and provide insight into why it outperforms OPTQ (\cref{OPTQ}).

Although the implementation of \cref{Qronos++} in \cite{zhang2025qronos} applies dampening to the Hessian via the regularization term $\widetilde{X}^\top \widetilde{X} + \lambda I$, we set $\lambda = 0$ for simplicity. The analysis, however, readily extends to the case $\lambda > 0$ using techniques similar to those in \cref{sec:OPTQ l2 error bound} and \cref{sec:OPTQ linf error bound}.

For brevity, we focus on a single neuron $w \in \mathbb{R}^N$ \new{in this section}. To begin the analysis, we note that \cite{zhang2025qronos} shows that \cref{Qronos++} is equivalent to iteratively running the following two steps for $t = 1, \dots, N$.
\begin{align}
    q_t &= \underset{p \in \mathcal{A}}{\mathrm{argmin}} \, \frac{1}{2} \| Xw - \sum_{j=1}^{t-1} q_j \widetilde{X}_j - p \widetilde{X}_t - \sum_{j=t+1}^{N} w^{(t-1)}_j \widetilde{X}_j \|_2^2, \label{original_bidq_update_q} \\
    w^{(t)}_{\geq t+1} &= \underset{(v_{t+1}, \dots, v_N) \in \mathbb{R}^{N - t}}{\mathrm{argmin}} \, \frac{1}{2} \| Xw - \sum_{j=1}^{t} q_j \widetilde{X}_j - \sum_{j=t+1}^{N} v_j \widetilde{X}_j \|_2^2. \label{original_bidq_update_w}
\end{align}
From the above formulation, it is natural to define the error at step $t$ as $e_t=Xw-\sum_{j=1}^{t}q_j \widetilde{X}_{j}-\sum_{j=t+1}^{N}w^{(t)}_{j}\widetilde{X}_{j}$. As a counterpart to \cref{lemma:OPTQ_error_recur}, the next lemma characterizes how $e_t$ evolves.
\begin{lemma}\label{lemma: e_formula_wl_error_correction}
    Running Qronos (\cref{Qronos++}) with $\mathcal{Q}$ or $\mathcal{Q}_{stoc}$ on $w\in\mathbb{R}^{N}$ using calibration dataset $X\in\mathbb{R}^{m\times N}$ gives 
    \begin{gather*}
        e_N = P_{\widetilde{X}_{\geq 2}^{\perp}}P_{\widetilde{X}_{1}^{\perp}}(Xw-\widetilde{X}w) + \sum_{j=1}^{N}P_{\widetilde{X}_{\geq j+1}^{\perp}}r_{j}\widetilde{X}_{j},
    \end{gather*}
    where $r_j$ are rounding errors with absolute value bounded by $\delta/2$ when using the deterministic RTN operator $\mathcal{Q}$, and $\delta$ when using the stochastic RTN operator $\mathcal{Q}_{stoc}$.
\end{lemma}

\begin{proof}
    We have $e_0=Xw-\widetilde{X}w$ and
    \begin{align*}
        e_1&=Xw-q_1\widetilde{X}_1  - \sum_{j=2}^{N}w^{(1)}_{j}\widetilde{X}_{j}
        =P_{\widetilde{X}_{\geq 2}^{\perp}}(Xw - q_1\widetilde{X}_{1})
    \end{align*}
    by the definition of $w^{(1)}_{\geq 2}$ in \cref{original_bidq_update_w}. Define 
     $   \widetilde{w}:=\underset{v\in\mathbb{R}}{\mathrm{argmin}}\frac{1}{2}\|Xw-v\widetilde{X}_{1}-\sum_{j=2}^{N}w_{j}\widetilde{X}_{j}\|_2^2.
    $
    By the choice of $q_{1}$ in \cref{original_bidq_update_q}, we know $q_1=\mathcal{Q}(\widetilde{w})$. Then we have
    \begin{align*}
        e_1&=P_{\widetilde{X}_{\geq 2}^{\perp}}(Xw - q_1\widetilde{X}_{1})\\
        &=P_{\widetilde{X}_{\geq 2}^{\perp}}(Xw - q_1\widetilde{X}_{1} - \sum_{j=2}^{N}w_{j}\widetilde{X}_{j})\\
        &=P_{\widetilde{X}_{\geq 2}^{\perp}}(Xw - \widetilde{w}\widetilde{X}_{1} - \sum_{j=2}^{N}w_{j}\widetilde{X}_{j} + (\widetilde{w}-q_1)\widetilde{X}_{1})\\
        &=P_{\widetilde{X}_{\geq 2}^{\perp}}\left(P_{\widetilde{X}_{1}^{\perp}}\left(Xw  - \sum_{j=2}^{N}w_{j}\widetilde{X}_{j}\right) + (\widetilde{w}-q_1)\widetilde{X}_{1}\right)\\
        &=P_{\widetilde{X}_{\geq 2}^{\perp}}\left(P_{\widetilde{X}_{1}^{\perp}}\left(Xw  - \sum_{j=1}^{N}w_{j}\widetilde{X}_{j}\right) + (\widetilde{w}-q_1)\widetilde{X}_{1}\right)\\
        &=P_{\widetilde{X}_{\geq 2}^{\perp}}P_{\widetilde{X}_{1}^{\perp}}e_0 + P_{\widetilde{X}_{\geq 2}^{\perp}} r_1 \widetilde{X}_1,
    \end{align*}
    where $r_1$ is the rounding error. When $t\geq2$, we can similarly compute that
    \begin{gather*}
        e_t = P_{\widetilde{X}_{\geq t+1}^{\perp}}P_{\widetilde{X}_{t}^{\perp}}e_{t-1} + P_{\widetilde{X}_{\geq t+1}^{\perp}} r_t \widetilde{X}_t,
    \end{gather*}
    where $r_t$ is the rounding error at step $t$. Using this recursive formula and the fact that $e_1$ is perpendicular to $\widetilde{X}_{\geq2}$, we can see (e.g., by induction) that when $t\geq2$, $e_{t-1}$ is perpendicular to the column space of $\widetilde{X}_{t}$. Thus $P_{\widetilde{X}_{\geq t+1}^{\perp}}P_{\widetilde{X}_{t}^{\perp}}e_{t-1}=e_{t-1}$. Then the recursive formula becomes
    \begin{gather*}
        e_t = e_{t-1} + P_{\widetilde{X}_{\geq t+1}^{\perp}} r_t \widetilde{X}_t.
    \end{gather*}
    Combing with the fact that $e_{1} =P_{\widetilde{X}_{\geq 2}^{\perp}}P_{\widetilde{X}_{1}^{\perp}}e_0 + P_{\widetilde{X}_{\geq 2}^{\perp}} r_1 \widetilde{X}_1 $, we deduce
    \begin{gather*}
        e_N = P_{\widetilde{X}_{\geq 2}^{\perp}}P_{\widetilde{X}_{1}^{\perp}}(Xw-\widetilde{X}w) + \sum_{j=1}^{N}P_{\widetilde{X}_{\geq j+1}^{\perp}}r_{j}\widetilde{X}_{j}.
    \end{gather*}
\end{proof}
\begin{remark}[A variation to Qronos]
Notice that one can first  set $ w^{(0)} = \arg\min_{\widetilde{w}}\|Xw - \widetilde{X}\widetilde{w}\|^2$, and then proceed normally with calibration data $\widetilde{X}$, which gives an error $e_N = P_{\widetilde{X}^{\perp}}(Xw - \widetilde{X}w) + \sum_{j=1}^{N}P_{\widetilde{X}_{\geq j+1}^{\perp}}r_{j}\widetilde{X}_{j}$, where $r_j$ defined similarly as above.
\end{remark}
The following proposition provides a Euclidean error bound for Qronos. Here, we only focus on the case where Qronos is run with the deterministic RTN operator $\mathcal{Q}$. However, extending the result to the stochastic RTN operator $\mathcal{Q}_{\text{stoc}}$ is straightforward and only entails replacing $\delta/2$ by $\delta$.

\begin{proposition}\label{prop:Qronos_deter_error}
    Running Qronos (\cref{Qronos++}) with deterministic RTN operator $\mathcal{Q}$ on $w\in\mathbb{R}^{N}$ using $X\in\mathbb{R}^{m\times N}$, we have 
\begin{gather}\label{eq:e_norm_bound_qronos}
  \|Xw - \widetilde{X}q\|_2\leq \|P_{\widetilde{X}_{\geq 2}^{\perp}}P_{\widetilde{X}_{1}^{\perp}}(Xw-\widetilde{X}w)\|_2+\dfrac{\delta}{2}\sqrt{N} \min\left\{\max_{j}\|P_{\widetilde{X}_{\geq j+1}^{\perp}}\widetilde{X}_{j}\|_2, \dfrac{\Tr(\widetilde{X}^\top \widetilde{X})}{N}\right\}.
\end{gather}
\end{proposition}
\begin{proof}
From \cref{lemma: e_formula_wl_error_correction}, one has
\begin{align*}
     \|Xw - \widetilde{X}q\|_2 &= \|P_{\widetilde{X}_{\geq 2}^{\perp}}P_{\widetilde{X}_{1}^{\perp}}(Xw-\widetilde{X}w) + \sum_{j=1}^{N}P_{\widetilde{X}_{\geq j+1}^{\perp}}r_{j}\widetilde{X}_{j}\|_2\\&\leq \|P_{\widetilde{X}_{\geq 2}^{\perp}}P_{\widetilde{X}_{1}^{\perp}}(Xw-\widetilde{X}w) \|_2+ \|\sum_{j=1}^{N}P_{\widetilde{X}_{\geq j+1}^{\perp}}r_{j}\widetilde{X}_{j}\|_2.
\end{align*}
    For $\|\sum_{j=1}^{N}P_{\widetilde{X}_{\geq j+1}^{\perp}}r_{j}\widetilde{X}_{j}\|_2$, similar to \cref{prop:OPTQ_deter_error}, one can first prove that $\{P_{\widetilde{X}_{\geq j+1}^{\perp}}\widetilde{X}_{j}\}_{j=1}^{N}$ are orthogonal to each other. Then
    \begin{gather*}
        \|\sum_{j=1}^{N}P_{\widetilde{X}_{\geq j+1}^{\perp}}r_{j}\widetilde{X}_{j}\|_2=\sqrt{\sum_{j=1}^{N}|r_j| \|P_{\widetilde{X}_{\geq j+1}^{\perp}}\widetilde{X}_{j}\|_2^2}.
    \end{gather*}
    Then one can use the fact that $|r_t| \leq \tfrac{\delta}{2}$ to finish the proof.
\end{proof}
The above proposition provides theoretical insight into why Qronos outperforms OPTQ, as observed in \cite{zhang2025qronos}. As noted in \cite{frantar2022gptq} and the corresponding GitHub repository\footnote{\url{https://github.com/IST-DASLab/gptq}}, in practice, the OPTQ algorithm (\cref{OPTQ}) is implemented using the activation $\widetilde{X}$ from the partially quantized neural network. In this case, by applying \cref{lemma:OPTQ_error_recur}, the OPTQ reconstruction error can be expressed as
\begin{equation}
e_{N}^{\text{OPTQ}} = Xw - \widetilde{X}q = (Xw - \widetilde{X}w) + (\widetilde{X}w - \widetilde{X}q) = Xw - \widetilde{X}w + \sum_{j=1}^{N} P_{\widetilde{X}_{\geq j+1}^{\perp}} r_{j} \widetilde{X}{j},
\end{equation}
where $r_j$'s are rounding errors. One can compare the OPTQ error, $e^{\text{OPTQ}}_N$, with the Qronos error, $e_N$, given in \cref{lemma: e_formula_wl_error_correction}. The main difference, apart from the (bounded) $r_j$ terms being different, lies in the first term: for OPTQ, it is $Xw - \widetilde{X}w$, while for Qronos, it is
\begin{equation}
P_{\widetilde{X}_{\geq 2}^{\perp}} P_{\widetilde{X}_{1}^{\perp}} (Xw - \widetilde{X}w).
\end{equation}
Applying a similar analysis as in \cref{prop:Qronos_deter_error}, we obtain the bound
\begin{gather*}
\left\|e^{\text{OPTQ}}_N \right\|_2 \leq  \left\|Xw - \widetilde{X}w\right\|_2 + \dfrac{\delta}{2}\sqrt{N} \min\left\{\max_{j}\|P_{\widetilde{X}_{\geq j+1}^{\perp}}\widetilde{X}_{j}\|_2, \dfrac{\Tr(\widetilde{X}^\top \widetilde{X})}{N}\right\}
\end{gather*}

Compared to OPTQ, the first term in the Qronos error in \cref{prop:Qronos_deter_error} is reduced by two successive projections onto $\widetilde{X}_{1}^{\perp}$ and $\widetilde{X}_{\geq 2}^{\perp}$. Consequently, its $\ell_2$ norm is typically significantly smaller, as the projection restricts the error to a subspace of much lower dimension. Moreover, when the quantized input $\widetilde{X}$ is low rank and in general position, the first term of the Qronos error vanishes entirely. This offers a theoretical explanation for the observed performance advantage of Qronos over OPTQ.

As with OPTQ, one can derive infinity norm bounds for Qronos with the stochastic quantizer $\mathcal{Q}_{stoc}$.
\begin{theorem}\label{BiDQ_with_ec}
    Running Qronos (\cref{Qronos++}) with stochastic RTN operator $\mathcal{Q}_{stoc}$, on $w\in\mathbb{R}^{N}$ using  $X\in\mathbb{R}^{m\times N}$, we have
    \begin{align*}
		&\max_{i}\left|\left(Xw - \widetilde{X}q\right)_{i} \right|\\
        &\leq  \max_{i}\left|\left(P_{\widetilde{X}_{\geq 2}^{\perp}}P_{\widetilde{X}_{1}^{\perp}}(Xw-\widetilde{X}w)\right)_{i}\right|+\delta \sqrt{2\pi p \log{N}}\max_{j}\|P_{\widetilde{X}^\perp_{\geq j}}\widetilde{X}_{j}\|_2
	\end{align*}
with probability greater than  $ 1-\tfrac{\sqrt{2}m}{N^p}$. 
\end{theorem}
\begin{proof}
    Consider the partial error $e_t$ associated with $w$:
    \begin{gather*}
        e_t = Xw - \sum_{j=1}^{t}q_j\widetilde{X}_{j} = \sum_{j=t+1}w^{(t)}_{j}\widetilde{X}_{j}.
    \end{gather*}From the proof of \cref{lemma: e_formula_wl_error_correction}, we have
    \begin{gather*}
        e_{1}= P_{\widetilde{X}_{\geq 2}^{\perp}}P_{\widetilde{X}_{1}^{\perp}}e_0 + P_{\widetilde{X}_{\geq 2}^{\perp}} r_1 \widetilde{X}_1 
    \end{gather*}
    and
    \begin{gather*}
        e_t = e_{t-1} + P_{\widetilde{X}_{\geq t+1}^{\perp}} r_t \widetilde{X}_t.
    \end{gather*}
    Since starting from $t\geq 2$, the recursive formula is in the same form as in \cref{lemma:OPTQ_error_recur} and \cref{lemma:bound_one_neuron_OPTQ}, we can use the same method to deduce
    \begin{gather*}
        e_t \prec_{cx} \mathcal{N}(P_{\widetilde{X}_{\geq 2}^{\perp}}P_{\widetilde{X}_{1}^{\perp}}e_0, \Sigma),
    \end{gather*}
    where $\Sigma $ is as in \cref{lemma:bound_one_neuron_OPTQ},
    \begin{gather*}
        \Sigma=\dfrac{\pi\delta^{2}}{2}\sum_{j=1}^{N}P_{\widetilde{X}_{\geq j+1}^{\perp}}\widetilde{X}_{j}\widetilde{X}_{j}^{\top} P_{\widetilde{X}_{\geq j+1}^{\perp}},
    \end{gather*}
    with
    \begin{gather*}
		\Sigma\preceq \dfrac{\pi \delta^{2}}{2}\max_{j}\|P_{\widetilde{X}^\perp_{\geq j}}\widetilde{X}_{j}\|_2^{2} I.
	\end{gather*}
    Thus, using \cref{lemma:convex_order} \cref{trans},  we have
	\begin{gather*}
		Xw-\widetilde{X}q \prec_{cx} \mathcal{N}\Big(P_{\widetilde{X}_{\geq 2}^{\perp}}P_{\widetilde{X}_{1}^{\perp}}(Xw-\widetilde{X}w), \dfrac{\pi \delta^{2}}{2}\max_{j}\|P_{\widetilde{X}^\perp_{\geq j}}\widetilde{X}_{j}\|_2^{2}I\Big).
	\end{gather*}
Then by \cref{lemma:convex_order} \cref{lemma:cx-gaussian-tail} with $ \alpha=\delta\sqrt{2\pi p \log{N}}\max_{j}\|P_{X^\perp_{\geq j+1}}X_{j}\|_2 $ we obtain
\begin{align*}
		&\max_{i}\left|\left(Xw - \widetilde{X}q\right)_{i} - \left(P_{\widetilde{X}_{\geq 2}^{\perp}}P_{\widetilde{X}_{1}^{\perp}}(Xw-\widetilde{X}w)\right)_{i} \right|\\
        &\leq \delta \sqrt{2\pi p \log{N}}\max_{j}\|P_{\widetilde{X}^\perp_{\geq j}}\widetilde{X}_{j}\|_2
	\end{align*}
 with probability greater than $ 1-\tfrac{\sqrt{2}m}{N^p}$. We apply the triangle inequality to finish the proof.
\end{proof}

\begin{remark}[Insights and Practical Implications]
All our previous remarks regarding, for example, the choice of $\lambda$,  the ordering of columns, the rank of $X$, and the alphabet size apply in the case of Qronos as well. This is due to the fact that the Qronos error bounds fundamentally only differ from the OPTQ ones by significantly reducing the error associated with the mismatch between $X$ and $\widetilde{X}$, which agrees with the empirical evidence in \cite{zhang2025qronos}.
\end{remark}

\section{Conclusions and Future Works}
We presented a comprehensive theoretical analysis of OPTQ, a widely used post-training quantization algorithm. Our work provides both deterministic $\ell_2$ and stochastic $\ell_\infty$ bounds, offering new insights into the algorithm's {success} and practical implications. In particular, our $\ell_2$ analysis in \cref{sec:OPTQ l2 error bound} explains how {OPTQ} quantization error {relates to} the structure of the calibration data, and offers justifications for common practical heuristics such as feature reordering. To further {strengthen the theoretical bounds for} OPTQ, in \cref{sec:OPTQ linf error bound}, we introduced a stochastic variant that guarantees $\ell_\infty$ bounds with a finer control over entry-wise errors, which is important {in the low-bit regime and in ensuring the most probable tokens are reliably identified}. Finally, in \cref{sec:Qronos theory}, we extended our framework to analyze Qronos, a recent algorithm with state-of-the-art performance, and established new theoretical bounds that explain its {better performance when compared to previous methods}. 

\new{Our analysis focuses on the layer-wise calibration errors that motivate the design of OPTQ and Qronos, namely \eqref{optq_error} and \eqref{qronos_error}. Under added assumptions, for instance Lipschitz nonlinear activation functions, our results can be used to bound the calibration error induced by quantizing the weights of an $L$-layer MLP (see, e.g., \cite{zhang2023spfq, zhang2023post}). It should also be possible to extend this framework to transformer layers used in LLMs, and to loss functions such as the cross-entropy loss but we leave this to future work. By contrast, obtaining guarantees directly for performance metrics such as perplexity and zero-shot accuracy on downstream tasks remains an open problem.}
\new{That said, calibration error (which was the focus of this paper) is often a useful proxy for such empirical metrics, since smaller calibration error typically coincides with lower perplexity and higher zero-shot accuracy for LLMs.  Another important direction for future work is generalization, that is, understanding when error bounds on calibration data also control error on unseen samples. While \cref{remark:generalization} provides a step in this direction, characterizing generalization to downstream tasks remains a challenge well beyond model compression. Bridging this gap is therefore an important direction for future work.}



\section*{Acknowledgments}
We gratefully acknowledge partial support from the National Science Foundation via grant DMS-2410717. We thank Johann Birnick for the insightful observation that our results lead to error bounds in the $\ell_2$ norm that are easy to interpret. We also thank Nick Fraser at AMD for helpful conversations.

\bibliographystyle{siamplain}
\bibliography{citations}

\appendix

\newpage
\section{Proof of lemmas for OPTQ Error Analysis}\label{sec:lemmas for OPTQ error}
The following lemma from  \cite{zhang2025qronos} shows that the OPTQ update  can be interpreted as the optimal adjustment of the remaining coordinates of $w$, chosen to best compensate—in a least-squares sense—for the quantization error introduced at the current step. 

\begin{lemma}[\cite{zhang2025qronos}]\label{thm:OPTQ1}
Equations \eqref{eq1} and \eqref{eq2} of OPTQ are equivalent to:
\begin{align}
    q_{t}&=\mathcal{Q}(w^{(t-1)}_t), \label{eq3}\\
    w^{(t)}_{\geq t+1}&=\underset{(v_{t+1},...,v_{N})\in\mathbb{R}^{N-t}}{\mathrm{argmin}}\frac{1}{2}\|(q_t - w^{(t-1)}_t)X_t + \sum_{j=t+1}^{N}(v_j - w^{(t-1)}_j)X_j\|_2^2.\label{eq4}
\end{align}
\end{lemma}

\subsection{Proof of \cref{lemma:OPTQ_error_recur}}\label{appendix:OPTQ_error_recur}
\paragraph{\cref{lemma:OPTQ_error_recur} (Restated)}
\new{\textit{Let $X \in \mathbb{R}^{m \times N}$ be full rank with $m \geq N$, and let $w \in \mathbb{R}^N$. Running OPTQ (\cref{OPTQ}) with $\lambda = 0$ (so $H = X^\top X$), the error defined in \eqref{eq:err_def} satisfies
\begin{align}
    e_t &= P_{X_{\geq t+1}^{\perp}}(w^{(t-1)}_t - q_t) X_t + e_{t-1} \text{\quad and \quad } e_N =\sum_{j=1}^{N} P_{X_{\geq j+1}^{\perp}}(w^{(j-1)}_j - q_j) X_j.\label{final error_appendix}
\end{align}
Moreover, the resulting quantized vector $q$ satisfies
\begin{gather}\label{eq:e_norm_appendix}
\|Xw - Xq\|_2^2 = \sum_{j=1}^N |w^{(j-1)}_j - q_j|^2 \, \|P_{X_{\geq j+1}^{\perp}} X_j\|_2^2. 
\end{gather}
In particular, this implies that when using the infinite alphabet $\mathcal{A}^\delta$
\begin{gather}\label{eq:e_norm_bound_appendix}
\|Xw - Xq\|_2 \leq \frac{\delta}{2} \sqrt{N} \cdot \min\left\{ \max_j \|P_{X_{\geq j+1}^\perp} X_j\|_2, \ \sqrt{\frac{\|X\|_F^2}{N}} \right\}.
\end{gather}}}
\begin{proof}
Recall we use $e_{t} $ to denote the  error at step $t$, where $e_{t}=Xw - \sum_{j=1}^{t}q_{j}X_{j} - \sum_{j=t+1}^{N}w^{(t)}_{j}X_{j}$. It is easy to observe that $e_{N}=Xw-Xq$ and $e_0 = 0$. We then have
\begin{equation}\label{eq:e_inductioN'}
\begin{aligned}
	e_{t}&=Xw - \sum_{j=1}^{t}q_{j}X_{j} - \sum_{j=t+1}^{N}w^{(t)}_{j}X_{j}\\
    &=Xw - \sum_{j=1}^{t-1}q_{j}X_{j} - (q_t - w^{(t-1)}_{t})X_{t}	- w^{(t-1)}_{t}X_{t}-\sum_{j=t+1}^{N}w^{(t)}_{j}X_{j}\\
    &=(w^{(t-1)}_{t}-q_t)X_{t} + \left(Xw - \sum_{j=1}^{t-1}q_{j}X_{j}-\sum_{j=t}^{N}w_{j}^{(t-1)}X_{j}\right)+\sum_{j=t+1}^{N}(w^{(t-1)}_{j}-w^{(t)}_{j})X_{j}\\
    &=(w^{(t-1)}_{t}-q_t) X_{t}+ \sum_{j=t+1}^{N}(w^{(t-1)}_{j}-w^{(t)}_{j})X_{j}+e_{t-1}.
\end{aligned}
\end{equation}
By \cref{eq4}, $ (w^{(t)}_{t+1},...,w^{(t)}_{N})^{\top} $ is chosen such that
\begin{align*}
\|(w^{(t-1)}_t-q_{t})X_{t}+\sum_{j=t+1}^{N}&({w}^{(t-1)}_{j}-w^{(t)}_{j})X_{j}\|_2^2 	=\\ &\underset{(v_{t+1},...,v_{N})\in\mathbb{R}^{N-t}}{\mathrm{min}} \|(w^{(t-1)}_t-q_t)X_{t}+\sum_{j=t+1}^{N}(w^{(t-1)}_{j}-v_{j})X_{j}\|_2^2 .
\end{align*}
So, \(
(w^{(t-1)}_{t}-q_t) X_{t}+ \sum_{j=t+1}^{N}(w^{(t-1)}_{j}-w^{(t)}_{j})X_{j}=P_{X_{\geq t+1}^{\perp}}(X_{t}(w^{(t-1)}_t-q_t)).
\)
Combining this and \cref{eq:e_inductioN'}, we deduce
\begin{align*}
e_{t}&=(w^{(t-1)}_{t}-q_t) X_{t}+ \sum_{j=t+1}^{N}(w^{(t-1)}_{j}-w^{(t)}_{j})X_{j}+e_{t-1}
=P_{X_{\geq t+1}^{\perp}}(X_{t}(w^{(t-1)}_t-q_t))+e_{t-1}.
\end{align*}
This gives 
\begin{gather*}
e_{t}=P_{X_{\geq t+1}^{\perp}}(X_{t}(w^{(t-1)}_t-q_t))+e_{t-1},
\end{gather*}
  which when applied recursively yields
  \begin{gather*}
    e_N=\sum_{j=1}^{N}P_{X_{\geq j+1}^{\perp}}(w^{(j-1)}_j-q_j)X_{j}.
\end{gather*}
For \cref{eq:e_norm_appendix}, it suffices to show for $i\neq j$, $P_{X_{\geq j+1}^{\perp}}X_{j}$ is orthogonal to $P_{X_{\geq i+1}^{\perp}}X_{i}$.
     Let $v_j =P_{X_{\geq j+1}^{\perp}}X_{j} $. Without loss of generality let $1\leq i < j \leq N$. We have
 \begin{align*}
     \langle v_{i}, v_j \rangle &= \langle P_{X_{\geq i+1}^{\perp}}X_{i}, P_{X_{\geq j+1}^{\perp}}X_{j}\rangle\\
     &= \langle  P_{X_{\geq j+1}^{\perp}}P_{X_{\geq i+1}^{\perp}}X_{i},X_{j}\rangle\\
     &=\langle P_{X_{\geq i+1}^{\perp}}X_{i},X_{j}\rangle\\
     &=\langle X_{i},P_{X_{\geq i+1}^{\perp}}X_{j}\rangle=0.
 \end{align*}
 The third equality is because $X_{\geq i+1}^{\perp} \subseteq X_{\geq j+1}^{\perp}$, and the last equality is due to the fact that $X_{j} \in \text{span}(X_{\geq i+1})$. This proves \cref{eq:e_norm_appendix}. Then \cref{eq:e_norm_bound_appendix} is due to the simple observation that $|(w^{(j-1)}_j-q_j)| \leq \dfrac{\delta}{2} $.
\end{proof}

\subsection{Proof of \cref{lemma:bound_one_neuron_OPTQ}}
\label{proof:bound_one_neuron_OPTQ}
\paragraph{\cref{lemma:bound_one_neuron_OPTQ} (Restated)}
\new{\textit{Let $q$ be the output of quantizing $w$ with OPTQ with stochastic quantizer $\mathcal{Q}_{stoc}$, then $ Xw - Xq \prec_{cx} \mathcal{N}(0, \Sigma) $, where \begin{align*}
 	 \Sigma&=\dfrac{\pi\delta^{2}}{2}\sum_{j=1}^{N}P_{X_{\geq j+1}^{\perp}}X_{j}X_{j}^{\top} P_{X_{\geq j+1}^{\perp}}\preceq \dfrac{\pi \delta^{2}}{2}\max_{j}\|P_{X^\perp_{\geq j+1}}X_{j}\|_2^{2} I.
 \end{align*}}}
\begin{proof}
We use induction to prove the lemma. The induction hypothesis is 
\begin{gather*}
	e_{t}\prec_{cx}\mathcal{N}(0, \Sigma_{t}),
\end{gather*}
where $ \Sigma_{t} $ is defined inductively with $ \Sigma_{0}=0 $ and 
\begin{gather*}
	\Sigma_{t}=\dfrac{\pi\delta^{2}}{2}P_{X_{\geq t+1}^{\perp}}X_{t}X_{t}^{\top}P_{X_{\geq t+1}^{\perp}}+\Sigma_{t-1}.
\end{gather*}
The base case $ e_{0}=0 \prec_{cx}\mathcal{N}(0, 0)$ is obvious.  Now assume the proposition is true for $t-1 $. Then, by  \cref{lemma:OPTQ_error_recur}, we have
\begin{gather*}
	e_{t}=P_{X_{\geq t+1}^{\perp}}(X_{t}(w^{(t-1)}_t-q_t)) + e_{t-1}.
\end{gather*}
Further, we observe that $e_{t-1}$ and the quantized values $q_1, \dots, q_{t-1}$ determine each other uniquely. First, if $q_1, \dots, q_{t-1}$ are fixed, then $e_{t-1}$ is also fixed due to the update rule in \cref{eq4} and the definition of $e_{t-1}$. Conversely, if $e_{t-1}$ is fixed, then from \cref{lemma:OPTQ_error_recur}, we have
\begin{gather*}
    e_{t-1}=\sum_{j=1}^{t-1}P_{X_{\geq j+1}^{\perp}}(X_{j}(w^{(j-1)}_j-q_j)).
\end{gather*}

In the proof of \cref{prop:OPTQ_deter_error}, it was shown that the terms in this sum are mutually orthogonal. Thus, by taking inner products with the deterministic vectors $P_{X_{\geq j+1}^{\perp}}(X_j)$ for $j = 1,...,t - 1$, we can recover each rounding error $w^{(j-1)}_j - q_j$. Starting with $w^{(0)}_1 - q_1$, we can recover $q_1$, which allows us to compute $w^{(1)}_{\geq 2}$. Then, using $w^{(1)}_2 - q_2$ and $w^{(1)}_{\geq 2}$, we can recover $q_2$. Repeating this process iteratively, we can reconstruct all $q_1, \dots, q_{t-1}$. Therefore, conditioning on the random variable $e_{t-1}$ is equivalent to conditioning on the random variables $q_1, \dots, q_{t-1}$. Based on this key observation, we notice that $$ \left(w^{(t-1)}_t-q_t\mid e_{t-1}\right)=\left(w^{(t-1)}_t-\Q(w^{(t-1)}_t)\mid e_{t-1}\right)\sim_{D} \left(w^{(t-1)}_t-\Q(w^{(t-1)}_t)\mid q_1,\dots,q_{t-1}\right) $$ is mean zero and bounded by $ \delta $. 
As a result, the only source of randomness arises from the stochastic nature of the RTN operator $\mathcal{Q}_{\text{stoc}}$. Then by \cref{lemma:convex_order}, \cref{lemma:cx-bounded}, we know 
\begin{gather*}
	w^{(t-1)}_t-q_t \mid e_{t-1}\prec_{cx} \mathcal{N}\Big(0, \dfrac{\pi \delta^{2}}{2}\Big).
\end{gather*}
Next by \cref{lemma:convex_order}, \cref{linear}, we obtain
\begin{gather*}
	P_{X_{\geq t+1}^{\perp}}(X_{t}(w^{(t-1)}_t-q_t)) \mid e_{t-1} \prec_{cx} \mathcal{N}\Big(0,\dfrac{\pi\delta^{2}}{2}P_{X_{\geq t+1}^{\perp}}X_{t}X_{t}^{\top} P_{X_{\geq t+1}^{\perp}}\Big).
\end{gather*}
But the induction hypothesis yields
\begin{gather*}
	e_{t-1}\prec_{cx}\mathcal{N}(0,\Sigma_{t-1}),
\end{gather*}
so by \cref{lemma:convex_order}, \cref{lemma:cx-sum} with $U=e_{t-1}$, $V-U=P_{X_{\geq t+1}^{\perp}}(X_{t}(w^{(t-1)}_t-q_t))$, $E=\mathcal{N}(0,\Sigma_{t-1})$ and $F=\mathcal{N}\Big(0,\dfrac{\pi\delta^{2}}{2}P_{X_{\geq t+1}^{\perp}}X_{t}X_{t}^{\top} P_{X_{\geq t+1}^{\perp}}\Big)$, we have
\begin{gather*}
	e_{t}=P_{X_{\geq t+1}^{\perp}}(X_{t}(w^{(t-1)}_t-q_t)) + e_{t-1}\prec_{cx}\mathcal{N}\Big(0,\dfrac{\pi\delta^{2}}{2}P_{X_{\geq t+1}^{\perp}}X_{t}X_{t}^{\top} P_{X_{\geq t+1}^{\perp}}\Big) + \mathcal{N}(0,\Sigma_{t-1}),
\end{gather*}
where the two Gaussian distributions on the right hand side are independent of each other.
As a result,
\begin{gather*}
	e_{t}\prec_{cx}\mathcal{N}\Big(0,\dfrac{\pi\delta^{2}}{2}P_{X_{\geq t+1}^{\perp}}X_{t}X_{t}^{\top} P_{X_{\geq t+1}^{\perp}}+ \Sigma_{t-1}\Big)=\mathcal{N}(0,\Sigma_{t}).
\end{gather*}
This completes the induction. Then we have
\begin{gather*}
	Xw-Xq = e_{N} \prec_{cx}\mathcal{N}(0, \Sigma_{N}).
\end{gather*}
And by the definition of $ \Sigma_{N} $, we know
\begin{gather*}	\Sigma_{N}=\Sigma_{0}+\sum_{j=1}^{N}\dfrac{\pi\delta^{2}}{2}P_{X_{\geq j+1}^{\perp}}X_{j}X_{j}^{\top} P_{X_{\geq j+1}^{\perp}}=\dfrac{\pi\delta^{2}}{2}\sum_{j=1}^{N}P_{X_{\geq j+1}^{\perp}}X_{j}X_{j}^{\top} P_{X_{\geq j+1}^{\perp}}=\Sigma.
\end{gather*}
This completes the proof of the covariance calculation. We now proceed to its upper bound.

 \(
 	 \Sigma=\dfrac{\pi\delta^{2}}{2}\sum_{j=1}^{N}P_{X_{\geq j+1}^{\perp}}X_{j}X_{j}^{\top} P_{X_{\geq j+1}^{\perp}},
 \)
 so it is a sum of $N$ rank 1 matrices of the form $P_{X_{\geq j+1}^{\perp}}X_{j}X_{j}^{\top} P_{X_{\geq j+1}^{\perp}}$. Let $v_j =P_{X_{\geq j+1}^{\perp}}X_{j} $. From the proof of \cref{prop:OPTQ_deter_error}, we know $\{v_j\}_{j=1}^N$ are mutually orthogonal. Thus $v_1,\dots,v_N$ form a complete set of eigenvectors of $\Sigma$ as $\Sigma = \sum_{j=1}^N v_j v^\top_{j}$ and $\{v_j\}_{j=1}^{N}$ are mutually orthogonal. Their corresponding eigenvalues are $\|v_j\|_2^2$. As a result, $\|\Sigma\|_{\textrm{op}} = \max_{j}\|v_j\|_2^2=  \max_{j} \|P_{X^\perp_{\geq j+1}} X_{j}\|_2$. This completes the proof.
\end{proof}

\section{Auxiliary Lemmas}
{
\begin{lemma}\label{lemma:projection_upper_bound}
    Suppose $X\in\mathbb{R}^{m\times N}$. Let $\widehat{X}$ be the matrix $\begin{pmatrix}
    X \\ 
    \sqrt{\lambda} I
\end{pmatrix}$ and $\sigma^{(j)}_{\min}$ be the smallest singular value of $X_{\geq j+1}$. Then
\begin{align*}
      \|P_{\widehat{X}^\perp_{\geq j+1}}\widehat{X}_{j}\|_2^2 \leq \begin{cases}
          \frac{\lambda}{(\sigma^{(j)}_{\min})^2+\lambda} \cdot \|X_{j}\|_2^2 + \lambda&\text{ when $m\leq N-j$}\\
          \|X_j\|_2^2 + \lambda&\text{ when $m > N-j$}
      \end{cases}.
\end{align*}
\end{lemma}
\begin{proof}
    Let $X_{\geq j+1} = U^{(j)} \Sigma^{(j)} V^{(j)\top}$ be the full SVD of ${X}_{\geq j+1},$ where $U^{(j)}\in\mathbb{R}^{m \times m}$, $\Sigma^{(j)}\in\mathbb{R}^{m\times (N-j)}$ and $V^{(j)}\in \mathbb{R}^{(N-j) \times (N-j)}$. Since $\|P_{\widehat{X}^\perp_{\geq j+1}}\widehat{X}_{j}\|_2^2 = \min_{b\in\mathbb{R}^{N-j}} \|\widehat{X}_{j} - \widehat{X}_{\geq j+1} b\|_2^2$, solving this $\ell_2$ minimization problem yields\begin{gather*}
    \|P_{\widehat{X}^\perp_{\geq j+1}}\widehat{X}_{j}\|_2^2 = X_j^\top (I - X_{\geq j+1}(X_{\geq j+1}^\top X_{\geq j+1} + \lambda I)^{-1} X_{\geq j+1}^\top)X_j + \lambda.
\end{gather*}
Then, using the SVD of $X_{\geq j+1}$, we have
\begin{align*}
      \|P_{\widehat{X}^\perp_{\geq j+1}}\widehat{X}_{j}\|_2^2 &= X_j^\top U^{(j)}(I -\Sigma^{(j)}(\Sigma^{(j)\top}\Sigma^{(j)} + \lambda I)^{-1}\Sigma^{(j)\top})U^{(j)\top}X_j + \lambda\\
      &\leq \|I -\Sigma^{(j)}(\Sigma^{(j)\top}\Sigma^{(j)} + \lambda I)^{-1}\Sigma^{(j)\top}\|_{op}\cdot \|X_j\|_2^2 + \lambda.
\end{align*}

In the case when $m > N-j$, let $s^{(j)}=(\sigma^{(j)}_{1}, \dots, \sigma^{(j)}_{N-j}) $ are the singular values of $X_{\geq j+1}$ in increasing order. we have $\Sigma^{(j)}=\begin{pmatrix}
     \text{diag} (s^{(j)})\\ 
   0
\end{pmatrix}.$ Then one can compute 
\begin{align*}
    \Sigma^{(j)}(\Sigma^{(j)\top}\Sigma^{(j)} + \lambda I)^{-1}\Sigma^{(j)\top}&=I -\begin{pmatrix}
        \text{diag}(r^{(j)})  & 0\\
        0&0
    \end{pmatrix}.
\end{align*}
where $$r^{(j)}=\left(\frac{(\sigma^{(j)}_{1})^2}{(\sigma^{(j)}_1)^2 + \lambda},\dots,\frac{(\sigma^{(j)}_{N-j})^2}{(\sigma^{(j)}_{N-j})^2 + \lambda}\right).$$ Then 
\begin{gather*}
    \left\|I -  \Sigma^{(j)}(\Sigma^{(j)\top}\Sigma^{(j)} + \lambda I)^{-1}\Sigma^{(j)\top}\right\|_{op}=\left\|I -\begin{pmatrix}
        \text{diag}(r^{(j)})  & 0\\
        0&0
    \end{pmatrix}\right\|_{op}=1.
\end{gather*}

In the case when $m \leq N-j$, let $s^{(j)}=(\sigma^{(j)}_{1}, \dots, \sigma^{(j)}_{m}) $ are the singular values of $X_{\geq j+1}$ in increasing order. we have $\Sigma^{(j)}=\begin{pmatrix}
     \text{diag} (s^{(j)})& 
   0
\end{pmatrix}.$ Then similarly, one has 
\begin{gather*}
    \left\|I -  \Sigma^{(j)}(\Sigma^{(j)\top}\Sigma^{(j)} + \lambda I)^{-1}\Sigma^{(j)\top}\right\|_{op}=\left\|I -\begin{pmatrix}
        \text{diag}(r^{(j)})
    \end{pmatrix}\right\|_{op}=\frac{\lambda}{(\sigma^{(j)}_1)^2+\lambda}.
\end{gather*}

Combining the above two cases, one can deduce
\begin{align*}
      \|P_{\widehat{X}^\perp_{\geq j+1}}\widehat{X}_{j}\|_2^2 &\leq \|I -\Sigma^{(j)}(\Sigma^{(j)\top}\Sigma^{(j)} + \lambda I)^{-1}\Sigma^{(j)\top}\|_{op}\cdot \|X_j\|_2^2 + \lambda\\
      &\leq \begin{cases}
          \frac{\lambda}{(\sigma^{(j)}_{\min})^2+\lambda} \cdot \|X_{j}\|_2^2 + \lambda&\text{ when $m\leq N-j$}\\
          \|X_j\|_2^2 + \lambda&\text{ when $m > N-j$}
      \end{cases}.
\end{align*}
\end{proof}
}
\begin{lemma}\label{lem:gen_pos}
    Let $X\in\mathbb{R}^{m\times N}$ be a matrix that is in general position with $m < N$. Use $\sigma^{(j)}_{\min}$ to denote the smallest singular value of $X_{\geq j+1}$. Then the sequence $\sigma_{\min}^{(j)}$ is decreasing in~$j$ when $m\leq N - j$.
\end{lemma}
\begin{proof}
    We compare $\sigma^{(j-1)}_{\min}$ and $\sigma^{(j)}_{\min}$ for $1 \leq j \leq N-m$. Since $1 \leq j \leq N-m$, both $X_{\geq j}$ and $X_{\geq j+1}$ are full-rank (of rank $m$). By definition, $\sigma^{(j-1)}_{\min}$ and $\sigma^{(j)}_{\min}$ are the smallest non-zero eigenvalues of $X_{\geq j}^\top X_{\geq j}$ and $X_{\geq j+1}^\top X_{\geq j+1}$, respectively. Notice that $X_{\geq j}^\top X_{\geq j}$ and $X_{\geq j} X_{\geq j}^\top$ share the same non-zero eigenvalues. Then we know $\sigma^{(j-1)}_{\min}$ is the smallest eigenvalue of $X_{\geq j} X_{\geq j}^\top$. This is because $X_{\geq j} X_{\geq j}^\top$ is invertible due to the fact that $X_{\geq j}$ is full row rank. Similarly, $\sigma^{(j)}_{\min}$ is the smallest eigenvalue of $X_{\geq j+1} X_{\geq j+1}^\top$. For any $z\in\mathbb{R}^{m}$, we have
    \begin{gather*}
        z^\top X_{\geq j+1} X_{\geq j+1}^\top z = \sum_{t=j+1}^{N} z^\top X_{t}X_{t}^\top z \leq \sum_{t=j}^{N} z^\top X_{t}X_{t}^\top z = z^\top X_{\geq j} X_{\geq j}^\top z.
    \end{gather*}
    Thus 
    \begin{gather*}
        \sigma^{(j)}_{\min} = \min_{\|z\|=1}z^\top X_{\geq j+1} X_{\geq j+1}^\top z \leq \min_{\|z\|=1}z^\top X_{\geq j} X_{\geq j}^\top z = \sigma^{(j-1)}_{\min}.
    \end{gather*}
\end{proof}

\section{Properties of Convex Ordering}\label{sec:cvx ord prop}

The following properties hold for convex ordering. Proofs can be found in \cite{alweiss2021discrepancy} and \cite{zhang2023spfq}.
\begin{lemma}\label{lemma:convex_order}
	\begin{enumerate}
		\item \label{trans} (Lemma 2.3 in \cite{alweiss2021discrepancy}) If $ X\prec_{cx}Y $ and $ Y\prec_{cx}Z $, then $ X\prec_{cx}Z $.
		\item \label{linear} (Lemma 2.4 in \cite{alweiss2021discrepancy})	If $ X\prec_{cx}Y $, then for any linear transformation $ M $ on $ \mathbb{R}^{n} $, we have $ MX\prec_{cx}MY $.
		\item \label{lemma:normal_dom} (Lemma A.2 in \cite{zhang2023spfq})	If $ A$ and $B$ are two positive semi-definite matrices and $ A \preceq B $, then $ \mathcal{N}(0, A)\prec_{cx} \mathcal{N}(0, B) $.
		\item \label{lemma:cx-sum} (Lemma 2.5 in \cite{alweiss2021discrepancy}) Consider random vectors $U$, $V$, $E$, and $F$. Let $U$ and $V$ live on the same probability space, and let $E$ and $F$ be independent. Suppose that $U\prec_\mathrm{cx} E$ and $(V-U)|U \prec_\mathrm{cx} F$. Then $V\prec_\mathrm{cx} E+F$.
		\item \label{lemma:cx-bounded} (Lemma 2.6 in \cite{alweiss2021discrepancy}) 	Let $X$ be a real-valued random variable with $\mathbb{E} X = 0$ and $|X|\leq C$. Then $X\prec_\mathrm{cx}\mathcal{N}\bigl(0, \frac{\pi C^2}{2}\bigr)$.
		\item \label{lemma:cx-gaussian-tail} (Lemma B.2 in \cite{zhang2023spfq})	Let $X$ be an $n$-dimensional random vector such that $X\prec_\mathrm{cx}\mathcal{N}(\mu, \sigma^2 I)$, and let $\alpha>0$. Then 
		\[
		\mathbb{P}\bigl(\|X-\mu\|_\infty \leq \alpha \bigr) \geq 1- \sqrt{2} n e^{-\frac{\alpha^2}{4\sigma^2}}.
		\]
	\end{enumerate}
\end{lemma}

\section{Adversarial Constructions for OPTQ}\label{Appendix:adversarial}
\new{In this section, we present two adversarial constructions for OPTQ: one for the case $\lambda>0$ with a reasonably chosen regularization parameter (\cref{adv lambda greater 0}), and one for the case $\lambda=0$ (\cref{adv lambda eq 0}). The auxiliary results needed for these constructions are presented in \cref{adv:aux}.}
\subsection{Auxiliary Results}\label{adv:aux}
\new{We begin with a lemma that gives a useful property of OPTQ in terms of the pre-computed Hessian matrix $H$ (see also \cite{quip}).}
\begin{lemma}\label{lem:OPTQ_new_w_updates}
\new{Let $H\in\mathbb{R}^{N\times N}$ be a non-singular matrix, and let $
H = R^\top R $ be a factorization of $H$, where $R$ is lower triangular. Then, running OPTQ (\cref{OPTQ}) with the given Hessian $H$ and weight vector $w\in \mathbb{R}^N$, one has
\begin{gather*}
w^{(t-1)}_t = w_t + \frac{1}{R_{tt}} \sum_{j=1}^{t-1} R_{tj} (w_j - q_j).
\end{gather*}}
\end{lemma}
\begin{proof}
\new{Since $H=R^\top R$, running OPTQ with Hessian $H$ is equivalent to running OPTQ with data matrix $R$ and $\lambda=0$. By \cref{thm:OPTQ2}, and using that $R$ is lower triangular, we have
\begin{align*}
w^{(t-1)}_{\ge t}
&= \underset{(u_t,\dots,u_N)\in\mathbb{R}^{N-t+1}}{\arg\min}
\frac12\left\|Rw-\sum_{j=1}^{t-1}q_jR_j-\sum_{j=t}^N u_jR_j\right\|_2^2 \\
&= \underset{(u_t,\dots,u_N)\in\mathbb{R}^{N-t+1}}{\arg\min}
\sum_{s=t}^N
\left(
\sum_{j=1}^{t-1} R_{sj}(w_j-q_j)
+\sum_{k=t}^s R_{sk}(w_k-u_k)
\right)^2.
\end{align*}
For each $s\ge t$, the $s$-th term depends only on $u_t,\dots,u_s$. Since $R_{ss}\neq 0$, the minimization can be carried out recursively from $u_t$ to $u_{N}$. In particular, $w_t^{(t-1)}$ is determined by minimizing
\begin{gather*}
\arg\min_{u_t\in\mathbb{R}}
\left(
\sum_{j=1}^{t-1} R_{tj}(w_j-q_j)+R_{tt}(w_t-u_t)
\right)^2.
\end{gather*}
Solving this gives $
w_t^{(t-1)} = w_t + \frac{1}{R_{tt}}\sum_{j=1}^{t-1} R_{tj}(w_j-q_j).$ }
\end{proof}
\new{Next, we present a useful result that reformulates OPTQ as a fixed-point iteration.}
\begin{lemma}\label{lem:regularized_fixed_point}
\new{Let $H\in\mathbb{R}^{N\times N}$ be a non-singular matrix, and suppose that $
H = R^\top D R,$ where \(R\in\mathbb{R}^{N\times N}\) is lower triangular with ones on the diagonal, and
\(D=\operatorname{diag}(d_1,\dots,d_N)\) is positive diagonal. Let \(q\) be the output of
OPTQ (\cref{OPTQ}) applied to Hessian $H$ and weight vector \(w\in\mathbb{R}^N\), and let
\[
v_t:=w_t^{(t-1)}, \qquad t=1,\dots,N,
\]
denote the intermediate values before quantizing the \(t\)-th coordinate.
Then
\begin{gather}\label{fixed_pt_equation_forward}
v-q = R(w-q).
\end{gather}
On the other hand, if  $v,q,w\in\mathbb{R}^N$ satisfy
\begin{equation}\label{fixed_pt_equation_backwad}
\begin{aligned}
v - q &= R(w-q),\\
q &= \mathcal{Q}(v),
\end{aligned}
\end{equation}
then running OPTQ (\cref{OPTQ}) on  $H$, weight vector $w$ with $\lambda = 0$ yields $q$ as the quantized weight vector and $v$ as the intermediate vector.}
\end{lemma}

\begin{proof}
\new{We first prove that the OPTQ outputs $v$ and $q$ satisfy \eqref{fixed_pt_equation_forward}. Set $
\widehat R:=D^{1/2}R. $
Then $
H=\widehat R^\top \widehat R,$ and \(\widehat R\) is lower triangular with positive diagonal. Hence, by \cref{lem:OPTQ_new_w_updates}, for each \(t\),
\begin{gather}\label{v_eq}
v_t
=
w_t+\frac{1}{\widehat R_{tt}}\sum_{s<t}\widehat R_{ts}(w_s-q_s),
\qquad q_t=Q(v_t).
\end{gather}
Subtracting \(q_t\) and multiplying by \(\widehat R_{tt}\), we obtain
\begin{gather}\label{hat_R_eq}
\widehat R_{tt}(v_t-q_t)
=
\widehat R_{tt}(w_t-q_t)+\sum_{s<t}\widehat R_{ts}(w_s-q_s)
=
\sum_{s\le t}\widehat R_{ts}(w_s-q_s).
\end{gather}
Since \(\widehat R\) is lower triangular, this is exactly the \(t\)-th coordinate of $
\operatorname{diag}(\widehat R)(v-q)=\widehat R(w-q).$ Now \(\widehat R=D^{1/2}R\), and since \(R\) has ones on the diagonal, $
\operatorname{diag}(\widehat R)=D^{1/2}.$
Therefore, $
D^{1/2}(v-q)=D^{1/2}R(w-q).$ Since \(D^{1/2}\) is invertible, it follows that $
v-q=R(w-q).$}

\new{We now show that any pair $(v,q)$ satisfying \eqref{fixed_pt_equation_backwad} indeed corresponds to the output of OPTQ. By the update rule of OPTQ and the relation $q = \mathcal{Q}(v)$, it suffices to show that, for any given quantized entries $q_1,\dots,q_{t-1}$, one has $v_t = w_t^{(t-1)}$, where $v$ is the given vector and $w_t^{(t-1)}$ is the corresponding intermediate value generated by OPTQ. The argument essentially reverses the proof above.}

\new{If $(v,q)$ satisfies \eqref{fixed_pt_equation_backwad}, then $D^{1/2}(v-q) = D^{1/2}R(w-q)$. Using the same notation as above, namely $\widehat{R} = D^{1/2}R$, and noting that $\text{diag}(\widehat{R}) = D^{1/2}$, we obtain $\text{diag}(\widehat{R})(v-q) = \widehat{R}(w-q)$. Hence \eqref{hat_R_eq} also holds. Rearranging this identity yields \eqref{v_eq}. It then follows from \cref{lem:OPTQ_new_w_updates} that $v_t = w_t^{(t-1)}$. This completes the proof.}
\end{proof}
\subsection{An Adversarial Construction for OPTQ with $\lambda>0$}\label{adv lambda greater 0}
\new{In this section, we show that for any $\alpha>0$, there exist $X$ (defined in \eqref{choice of X}) and $w$ (defined in \eqref{choice of w}), with $\lambda=\alpha\|X\|_F^2/N$ as suggested in \cref{remark:lamda_suggestion}, such that when OPTQ is applied to $X$ and $w$ with parameter $\lambda$, it produces a quantized vector $q$ satisfying
\begin{gather*}
\|w-q\|_\infty \gtrsim \sqrt{N},
\qquad
\|X(w-q)\|_\infty \gtrsim \sqrt{N}
\end{gather*}
when $\alpha $ is small compared to $N$. This holds despite $$\|X\|_{\mathrm{op}} = O(1),\qquad\|w\|_\infty\le \frac12.$$ We emphasize that this is an existential statement for each fixed $\alpha$ thus it does not describe the behavior of OPTQ on a fixed instance as $\alpha$ varies.}

\new{We begin with the reparameterization $\lambda=\alpha\|X\|_F^2/N := (1-\rho)^2$ with $\rho \in(0,1)$. At the end of the proof, we will specify the choice of $\rho$ in terms of $\alpha$ and $N$. We first choose $w$, and then construct $X$. Let $u:=\frac{1}{\sqrt{N-1}}(1,\dots,1,0)^\top\in\mathbb{R}^N$, $R:=I-\rho e_Nu^\top$, and $\lambda:=(1-\rho)^2$. Since $u_N=0$, the matrix $R$ is lower triangular with ones on the diagonal. Moreover, $(\rho e_Nu^\top)^2=\rho^2 e_N(u^\top e_N)u^\top=0$, and hence one can verify $R^{-1}=I+\rho e_Nu^\top$.}

\new{We next show that $G:=R^\top R-\lambda I$ is positive-definite. Indeed, for every $x\in\mathbb{R}^N$, we have
\[
\|Rx\|_2=\|x-\rho e_Nu^\top x\|_2 \ge \|x\|_2-\rho\|e_Nu^\top x\|_2 \ge (1-\rho)\|x\|_2.
\]
It follows that $R^\top R \succeq (1-\rho)^2I=\lambda I$, so $G\succeq 0$. Let $B:=G^{1/2}$.}

\new{Now set $s:=\frac13(1,\dots,1,0)^\top$ and $z:=R^{-1}s$. Since $u^\top s=\frac{1}{\sqrt{N-1}}\cdot \frac{N-1}{3}=\frac{\sqrt{N-1}}{3}$, we obtain
\[
z=s+\rho e_Nu^\top s=\frac13(1,\dots,1,\rho\sqrt{N-1})^\top.
\]
In particular, $\|z\|_\infty=\frac{\rho}{3}\sqrt{N-1}$. Define $q:=-Q(z)$, $v:=q+s$, and
\begin{gather}\label{choice of w}
w:=z+q.
\end{gather}
Then $Q(v)=q$, since the first $N-1$ coordinates of $v$ equal the corresponding coordinates of $q$ plus $1/3$, while the last coordinate equals $q_N$. Also, $w-q=z$ and $v-q=s=Rz=R(w-q)$. Hence, by \Cref{lem:regularized_fixed_point}, $q$ is exactly the OPTQ output for the regularized matrix $\widehat X=\begin{bmatrix} X\\ \sqrt{\lambda}\,I \end{bmatrix}$, provided that $X^\top X=G$. From \eqref{choice of w}, we have
$$\|w-q\|_\infty = \|z\|_\infty = \frac{\rho}{3}\sqrt{N-1}.$$Since $w=z-Q(z)$, we also have $$\|w\|_\infty\le \frac12.$$}

\new{To obtain the desired $\ell_\infty$ lower bound for $X(w-q)$, choose an orthogonal matrix $U$ such that $UBz=\|Bz\|_2 e_1$, and define
\begin{gather}\label{choice of X}
X:=UB.
\end{gather}
Then $X^\top X=B^\top U^\top U B=B^2=G=R^\top R-\lambda I$. Moreover, $\|X\|_{\mathrm{op}}=\|B\|_{\mathrm{op}}$. Since $G\succeq 0$, we have
\[
\|X\|_{\mathrm{op}}^2=\|G\|_{\mathrm{op}}=\|R^\top R-\lambda I\|_{\mathrm{op}}=\|R\|_{\mathrm{op}}^2-\lambda \le (1+\rho)^2-(1-\rho)^2=4\rho,
\]
and therefore $\|X\|_{\mathrm{op}}\le 2\sqrt{\rho}$.}

\new{Finally, since $w-q=z$, we have $\|X(w-q)\|_\infty=\|Xz\|_\infty=\|Bz\|_2$. Also,
\[
\|Bz\|_2^2=z^\top(R^\top R-\lambda I)z=\|Rz\|_2^2-\lambda\|z\|_2^2=\|s\|_2^2-\lambda\|z\|_2^2.
\]
Since $\|s\|_2^2=\frac{N-1}{9}$ and $\|z\|_2^2=\frac{N-1}{9}(1+\rho^2)$, it follows that
\[
\|X(w-q)\|_\infty^2=\frac{N-1}{9}\bigl(1-(1-\rho)^2(1+\rho^2)\bigr),
\]
and hence
\[
\|X(w-q)\|_\infty=\frac{\sqrt{N-1}}{3}\sqrt{1-(1-\rho)^2(1+\rho^2)}.
\]}

\new{For this construction, we further have
\[
\|X\|_F^2=\tr(X^\top X)=\tr(R^\top R)-N\lambda=N(2\rho-\rho^2)+\rho^2.
\]
Hence, writing $\lambda=\alpha\|X\|_F^2/N$, we obtain
\[
\alpha=\frac{\lambda N}{\|X\|_F^2}
=\frac{(1-\rho)^2N}{N(2\rho-\rho^2)+\rho^2}
=\frac{(1-\rho)^2}{2\rho-\rho^2+\rho^2/N}.
\]
Therefore, for any given $\alpha$ and $N$, as the right hand side of the above expression ranges from 0 to $\infty$ as $\rho$ ranges from 1 to 0, one can solve for the corresponding $\rho$ and obtain the desired adversarial construction. Indeed, solving for $\rho$ yields $
\rho=\frac{(1+\alpha)-\sqrt{\alpha\left(1+\alpha+\frac{1}{N}\right)}}{1+\alpha-\alpha/N}.$
This completes the proof.}
\subsection{An Adversarial Construction for OPTQ with $\lambda = 0$}\label{adv lambda eq 0}
Here, we construct a matrix $X$ and vector $w$ so that OPTQ with a infinite alphabet results in $\|X(w-q)\|_\infty = \|X(w-q)\|_2 = O(\sqrt{N})$, and also $\|q\|_\infty = O(N), $ despite having $\|w\|_\infty \leq 1.$ 

Consider a matrix $X = H^\top R \in \R^{N \times N}$, where $H \in \R^{N \times N}$ is orthonormal and $R$ is a lower-triangular matrix with ones on the diagonal. By \cref{lem:regularized_fixed_point}, we know that the vector of weight updates produced by OPTQ, namely $v = (w_t^{(t-1)})_{t=1}^N$, satisfies the fixed-point equation
\begin{equation}
    v = R(w - Q(v)) + Q(v). \label{eq:OPTQ_fixed_point}
\end{equation}
Rearranging and recalling that $q = Q(v)$, we obtain
\[
X(w - q) = H^\top (v - q).
\]
In particular, if we choose $v - q = \beta H_j\in\R^N$ for some column index $j$ and scalar $\beta > 0$, then
$X(w - q) = \beta e_j,$ so that $\|X(w - q)\|_\infty = \|X(w - q)\|_2 = \beta.$
Assuming for simplicity that $\mathcal{A} = \mathbb{Z}$ (i.e., OPTQ uses unit step size), this setup can be realized as follows. First, we choose an arbitrary integer vector $q = Q(v) \in \mathbb{Z}^N$, and define
\(
v = q + \beta H_j.
\)
This choice is consistent with $q = Q(v)$ provided $\beta < \frac{1}{2\|H_j\|_\infty}$, ensuring rounding $v$ entrywise recovers $q$. Substituting into \eqref{eq:OPTQ_fixed_point} yields
\begin{equation}\label{eq:w_bad}
w = R^{-1}(v - q) + q = \beta R^{-1} H_j + q.
\end{equation}

Now, to construct an example where the OPTQ error scales poorly with $N$, we choose $H$ to be a bounded orthonormal system (see \cite{foucart2013mathematical}), such as the discrete cosine transform (DCT) matrix \cite{DCT} or a column-normalized Hadamard matrix. In either case, we have $\max_{i,j} |H_{i,j}| = O(1/\sqrt{N})$, and so choosing $\beta = \tfrac{1}{3\max_{i,j}|H_{i,j}|}$ gives
\[
\|X(w - q)\|_\infty = \|X(w - q)\|_2 = O(\sqrt{N}),
\]
even though we are using an infinite alphabet with step size $\delta = 1$.

 We now show that in such a construction, the gap $\|w - q\|_\infty$ can be made to scale as $O(N)$. From \eqref{eq:w_bad} we have $\|w - q\|_\infty = \beta \|R^{-1} H_j\|_\infty$. To make this large,  let $R$ be the lower-triangular matrix with ones on the diagonal and on the first sub-diagonal, and zero otherwise. let $H$ be a column-normalized Hadamard matrix, so that $\beta=\tfrac{\sqrt{N}}{3}$. In this setup, $R^{-1}$ is lower triangular with non-zero entries given by $R_{i,j} = (-1)^{i-j}, j\geq i$. These entries alternate in sign and match the sign pattern of $H_2$, the second column of the Hadamard basis. Then
$
R^{-1} H_2 = \left( \frac{1}{\sqrt{N}}, \frac{2}{\sqrt{N}}, \ldots, \frac{N}{\sqrt{N}} \right)^\top = \frac{1}{\sqrt{N}}(1, 2, \ldots, N)^\top,
$
and so
\(
w - q = \beta R^{-1} H_2 =\frac{1}{3} (1, 2, \ldots, N)^\top.
\)
This implies $$\|w - q\|_\infty = O(N).$$

To show that this can occur even with a small $\|w\|_\infty$, define $q = Q(v) = -Q(\beta R^{-1} H_j)$, so that $w = \beta R^{-1} H_j + q$ has $\|w\|_\infty < 1$. In contrast, $q$ has entries of magnitude $O(N)$, leading to a maximal $\ell_\infty$ distortion between $w$ and $q$ (thus necessitating a large alphabet).

\end{document}